\documentclass{article} 
\usepackage{iclr2026_conference,times}

\usepackage{amsmath,amsfonts,bm}

\def\eqref#1{equation~\ref{#1}}

\def\1{\bm{1}}

\DeclareMathAlphabet{\mathsfit}{\encodingdefault}{\sfdefault}{m}{sl}
\SetMathAlphabet{\mathsfit}{bold}{\encodingdefault}{\sfdefault}{bx}{n}

\usepackage{amsmath}
\usepackage{afterpage}
\usepackage{amsthm}
\usepackage{enumitem}
\newtheorem{theorem}{Theorem}
\usepackage{hyperref}
\usepackage{algorithm}
\usepackage[noend]{algpseudocode}
\usepackage[nameinlink,capitalise]{cleveref}
\usepackage[T1]{fontenc}
\usepackage[dvipsnames,table]{xcolor}
\usepackage{hyperref}
\usepackage{url}
\usepackage{booktabs}
\usepackage{graphicx}
\usepackage{subcaption}
\usepackage{booktabs}  
\usepackage{multirow}  
\usepackage{siunitx}   
\usepackage{makecell}
\usepackage{pifont}
\newcommand{\cmark}{\ding{51}}
\newcommand{\xmark}{\ding{55}}
\usepackage[normalem]{ulem}
\usepackage[most]{tcolorbox}
\usepackage{listings}
\usepackage[scaled=0.92]{inconsolata}

\newlength{\cafeArxivHeadheightDelta}
\fancypagestyle{cafeArxivFirst}{%
  \fancyhead{}%
  \fancyhead[L]{\includegraphics[height=20pt]{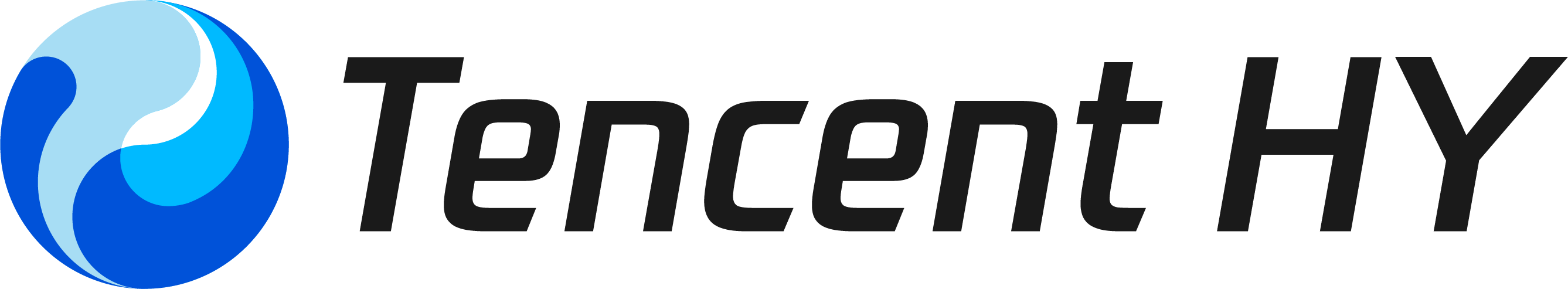}}%
  \fancyhead[R]{\includegraphics[height=24pt]{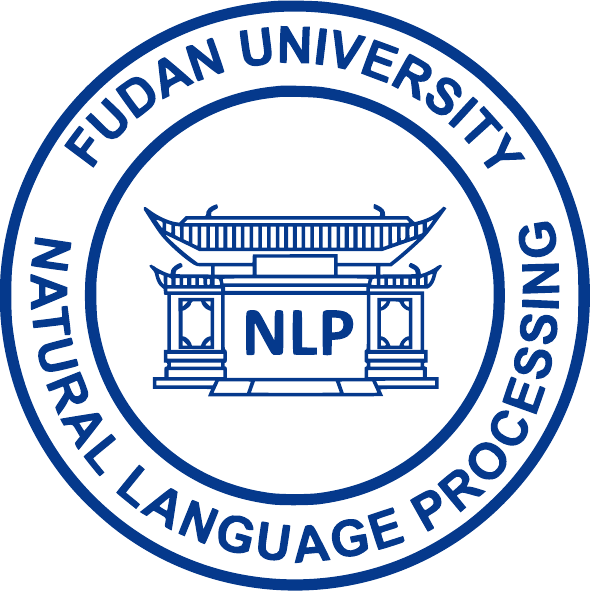}}%
}
\fancypagestyle{cafeArxivBody}{%
  \fancyhead{}%
}

\lstdefinestyle{cafePrompt}{
  basicstyle=\footnotesize\ttfamily,
  breaklines=true,
  columns=fullflexible,
  keepspaces=true,
  showspaces=false,
  breakindent=0pt,
  showstringspaces=false,
  aboveskip=0pt,
  belowskip=0pt
}
\lstdefinestyle{cafeCasePrompt}{
  style=cafePrompt,
  basicstyle=\fontsize{8}{8.8}\selectfont\ttfamily
}
\definecolor{Periwinkle}{rgb}{0.8, 0.8, 1.0}
\definecolor{CAFEBlue}{HTML}{5E81AC}
\definecolor{CAFEAmber}{HTML}{D9A23A}
\definecolor{CAFERed}{HTML}{C4554D}
\definecolor{CAFEGreen}{HTML}{5F9C72}
\definecolor{CAFEInk}{HTML}{344054}
\definecolor{CAFEGray}{HTML}{667085}

\tcbset{
  cafeCaseOuter/.style={
    enhanced,
    width=0.98\linewidth,
    colback=white,
    colframe=CAFEBlue!70!black,
    boxrule=0.7pt,
    arc=1.2mm,
    left=1.8mm,
    right=1.8mm,
    top=1.6mm,
    bottom=1.6mm,
    colbacktitle=CAFEBlue!8,
    coltitle=CAFEInk,
    fonttitle=\small\bfseries,
    toptitle=1.2mm,
    bottomtitle=1.2mm
  },
  cafeQueryBox/.style={
    enhanced,
    colback=CAFEBlue!4,
    colframe=CAFEBlue!55!black,
    boxrule=0.55pt,
    arc=0.8mm,
    left=1.4mm,
    right=1.4mm,
    top=1.1mm,
    bottom=1.1mm
  },
  cafeStageBox/.style={
    enhanced,
    height=10mm,
    valign=center,
    halign=center,
    boxrule=0.55pt,
    arc=0.8mm,
    left=0.6mm,
    right=0.6mm,
    top=0.5mm,
    bottom=0.5mm
  },
  cafeContentBox/.style={
    enhanced,
    boxrule=0.55pt,
    arc=0.8mm,
    left=1.4mm,
    right=1.4mm,
    top=1.1mm,
    bottom=1.1mm,
    coltitle=CAFEInk,
    fonttitle=\footnotesize\bfseries,
    toptitle=0.8mm,
    bottomtitle=0.8mm
  },
  cafeLaneBox/.style={
    enhanced,
    colback=white,
    boxrule=0.65pt,
    arc=1mm,
    left=1.3mm,
    right=1.3mm,
    top=1.2mm,
    bottom=1.2mm,
    coltitle=CAFEInk,
    fonttitle=\footnotesize\bfseries,
    toptitle=1mm,
    bottomtitle=1mm
  }
}

\newcommand{\cafetag}[2]{%
  \fcolorbox{#1!70!black}{#1!7}{\strut\scriptsize\bfseries #2}%
}
\title{CAFE: Self-Improving Search Agents Need \\Co-Evolving Feedback }

\author{\textbf{Boyang Liu}{\footnotesize $^{1,2}$}\thanks{Equal Contribution   \ \ $^{\dagger}$ Corresponding Author $^{\ddagger}$ Project Leader} 
 \;
 \textbf{Senjie Jin}{\footnotesize $^{1,2,*, \ddagger}$}
 \;
 \textbf{Peixin Wang}{\footnotesize $^{1,*}$}
 \;
 \textbf{Zhangyue Yin}{\footnotesize $^{2}$}
 \;
 \textbf{Yibo Wang}{\footnotesize $^{2}$}
 \;
 \\
 \textbf{Yuhao Zhou}{\footnotesize $^{1,2}$}
 \;
 \textbf{Zhihao Zhang}{\footnotesize $^{1}$}
 \;
 \textbf{Xinbing Liang}{\footnotesize $^{2}$}
 \;
  \textbf{Shizheng Zhu}{\footnotesize $^{1}$}
 \;
 \textbf{Yuhui Wang}{\footnotesize $^{1}$}
 \;
 \\
 \textbf{Jingqi Tong}{\footnotesize $^{1}$}
 \;
 \textbf{Dingwei Zhu}{\footnotesize $^{1}$}
 \;
 \textbf{Zhiheng Xi}{\footnotesize $^{1}$}
 \;
 \textbf{Jiazheng Zhang}{\footnotesize $^{1}$}
 \;
 \textbf{Clive Bai}{\footnotesize $^{2}$}
 \;
 \textbf{Clarenceai}{\footnotesize $^{2}$}
 \;
 \\
 \textbf{Blaze Chen}{\footnotesize $^{2}$}
 \;
 \textbf{Tao Gui}{\footnotesize $^{1,\dagger}$}
 \;
 \textbf{Qi Zhang}{\footnotesize $^{1}$}
 \;
  \textbf{Xuanjing Huang}{\footnotesize $^1$} \\[2pt]
{\footnotesize $^1$}Fudan University \;
  {\footnotesize $^2$}LLM Department, Tencent \; \\
  \texttt{\{boyangliu25,sjjin24\}@m.fudan.edu.cn \ tgui@fudan.edu.cn}\\
  \texttt{\{clivebai,blazeechen\}@tencent.com}\\
}

\iclrfinalcopy
\begin{document}

\maketitle
\thispagestyle{cafeArxivFirst}
\pagestyle{cafeArxivBody}

\begin{abstract}
Reliable search requires more than acquiring external evidence. An agent must
also recognize and recover from errors as its trajectory unfolds. In-trajectory
feedback provides a mechanism for such recovery by diagnosing where the search
has drifted and redirecting subsequent reasoning steps. This is
particularly important in long-horizon search, where an early directional error
may receive no immediate corrective signal and can compound across later steps.
Making such feedback learnable, however, creates a coupled problem: the agent
must learn when to request and use feedback, while the critic must learn
corrections from outcome-confounded rollouts as the agent's failure patterns
evolve. We introduce CAFE (Coupled Agent--Feedback Evolution), a framework in
which a shared-parameter model alternates between search-agent and critic roles.
CAFE initializes
feedback-conditioned recovery from trajectories built around the base agent's
own failures, then couples online and offline optimization. During online RL, a
comparative feedback estimate uses a prompt-level call--skip success gap to
shape request returns, while feedback-aware advantage shaping reweights token
advantages before and after feedback. Offline, rollout-derived preference optimization
learns feedback from matched successful and unsuccessful trajectories. On
seven agentic search benchmarks, CAFE outperforms the evaluated RL-based search
agents on average, retains its gains across all six out-of-domain benchmarks,
and reduces answer-level hallucinations. One-sided ablations show that
improving only the agent or only the critic eventually plateaus, whereas
alternating the two updates continues to improve performance. These findings
suggest that a self-improving search agent needs feedback that co-evolves with
the policy it guides.
\end{abstract}

\section{Introduction}
Search agents answer knowledge-intensive questions by iteratively interacting with external search environments, issuing queries and revising their behavior based on retrieved evidence \citep{xi2025survey, DBLP:conf/emnlp/LiDJZZZZD25, DBLP:conf/emnlp/ZhengFHCYLL25}. Many earlier retrieval-augmented pipelines relied on fixed strategies that determined when retrieval should occur, either at predetermined reasoning stages \citep{trivedi2023interleaving} or when a hand-crafted condition was triggered \citep{jiang2023active}. In contrast, outcome-supervised search agents learn policies for when and how to search from terminal task rewards \citep{jin2025searchr1trainingllmsreason, song2025r1searcherincentivizingsearchcapability}. 
Yet this autonomy remains largely outward-facing, with agents learning what external knowledge to seek while lacking comparable introspection into their own search trajectories.

In long-horizon search, an early directional error may receive no immediate corrective signal. Its cost is not confined to the errant step but propagates throughout the subsequent trajectory \citep{wang2026long, an2026erase, qi2026trajdebug}.
Nor is the source of failure easy to identify retrospectively. 
For instance, the agent may apply a constraint to the initial candidates but silently drop it in later steps, leave a required subgoal unexplored, or repeatedly rephrase a query without acquiring new evidence \citep{wong2026widesearch, wandr2026}.
In such scenarios, the trajectory begins correctly but ends in failure, and the terminal reward cannot say where. Recent work has therefore introduced finer-grained credit through information gain \citep{IGPO}, confidence change \citep{TIPS}, and local state comparison \citep{StepSearch, GiGPO}. These signals remain evaluative rather than instructive \citep{sutton1998reinforcement}, assigning credit retrospectively while deferring correction to future policy updates.
Redirecting the active trajectory instead requires in-context feedback that diagnoses where the search has drifted and prescribes what to try next.


Although prior work has explored in-context feedback, methods such as Reflexion \citep{reflexion}, Self-Refine \citep{Self-Refine}, and CRITIC \citep{CRITIC} typically return natural-language critiques post hoc and via prompted rather than trained models. Making such feedback learnable and delivering it within the active trajectory introduces three intertwined challenges. The first lies with the agent, which must learn when to request an intervention, while its learning objective must determine which behavior the resulting success should reinforce, 
given that a rescued trajectory and a flawless one yield identical rewards \citep{uesato2022solving}. 
The second lies with the critic, whose feedback must be learned without ground truth, from a reward confounded by both the preceding search and the agent’s subsequent actions.
The third emerges between them, as online updates shift the states the agent visits and the failures it encounters \citep{ackermann2025off}, potentially leaving a static critic misaligned with the evolving policy it is meant to guide.

To address these challenges, we introduce \textbf{CAFE} (\textbf{C}oupled \textbf{A}gent–\textbf{F}eedback \textbf{E}volution), an iterative, shared-parameter framework that integrates corrective feedback into the active search trajectory and alternates optimization of the agent and critic capabilities.
Learning this coupled interaction directly from sparse outcomes is difficult, so we initialize CAFE with recovery demonstrations 
constructed from the base agent’s own failures.
Rather than replacing a failed rollout with an ideal trajectory, we preserve its erroneous prefix, 
insert corrective feedback, and retain a successful continuation.
This teaches the model to request, generate, and use feedback at states its own policy actually visits.
Imitation alone, however, does not reveal when an intervention is useful or how the resulting success should be credited.
During online RL, the comparative feedback estimate (CFE) uses the empirical prompt-level success gap between rollouts with and without feedback requests to shape request returns, while feedback-aware advantage shaping redistributes credit across the intervention.
Offline, rollout-derived preference optimization (RDPO) learns from prefix-matched successful and failed trajectories, reducing outcome confounding. 
At each iteration, CAFE alternates online agent learning with offline critic refinement using the latest rollouts, maintaining alignment between the two capabilities as they co-evolve.

We evaluate CAFE using Qwen2.5-7/3B-Instruct \citep{qwen2025qwen25technicalreport} across seven agentic SearchQA benchmarks, with additional evaluation on BrowseComp-Plus \citep{chen2025browsecompplusfairtransparentevaluation}. 
At the 7B scale, CAFE achieves the strongest average performance among competing RL-based search methods, outperforming the strongest baseline by 2.1 EM and 1.3 F1. Notably, these gains are consistent across all six out-of-domain benchmarks and generalize robustly to the 3B scale.
Extensive component and objective ablations confirm the necessity of our feedback-aware credit assignment and the alternating optimization schedule. Furthermore, in-depth analyses of agent–critic cross-play, training dynamics, and hallucination (reducing the average answer-level rate from 17.6\% to 12.6\%) explicitly demonstrate how sustained co-evolution actively drives the observed performance improvements.

\noindent\textbf{Our contributions are fourfold:}
\begin{enumerate}[
    labelindent=0pt,
    leftmargin=*,
    labelsep=0.6em,
    itemsep=0.3em,
    topsep=0.3em
]
    \item[\ding{182}] \textbf{We formulate self-improving search as a coupled agent–feedback learning problem.} CAFE integrates both roles within a shared-parameter model, alternating online policy updates with offline critic refinement.
    \item[\ding{183}] \textbf{We develop a targeted RL framework for seeking and utilizing feedback.} Our comparative feedback estimate (CFE) and advantage shaping actively learn when to trigger interventions and how to assign credit.
    \item[\ding{184}] \textbf{We introduce rollout-derived preference optimization (RDPO) for feedback generation.} By learning from prefix-matched preference pairs mined from the latest online rollouts, RDPO keeps the critic aligned with the agent as it evolves.
    \item[\ding{185}] \textbf{We provide empirical evidence for co-evolution.} CAFE achieves the best average performance among the evaluated RL-based baselines, transfers robustly to all out-of-domain datasets, reduces hallucinations, and \textbf{drives sustained gains in both trajectory quality and search performance}.
\end{enumerate}




\section{Methodology}
\label{sec:method}
\textbf{Motivation.} 
While recent work has primarily advanced search agents’ ability to acquire external evidence, we argue that robust long-horizon search also requires active in-trajectory error diagnosis and correction, and as policy updates shift the agent’s state and failure distributions, the critic providing this guidance must adapt accordingly. Motivated by this coupling, we propose \textbf{CAFE}, an iterative, shared-parameter framework that integrates corrective feedback into active trajectories and alternates agent and critic optimization. 

We develop CAFE around three questions: \textbf{RQ1.} (\Cref{sec:shared_feedback_search}) How should in-trajectory feedback be structured and how can the agent learn to generate it?   \textbf{RQ2.} (\Cref{sec:online_agent_optimization}) How can an agent learn when to request feedback and how to utilize it? \textbf{RQ3.} (\Cref{sec:offline_feedback_training}) How can a shared model co-evolve its agent and critic capabilities?


\begin{figure}[t]
    \centering
    
    \includegraphics[width=\linewidth]{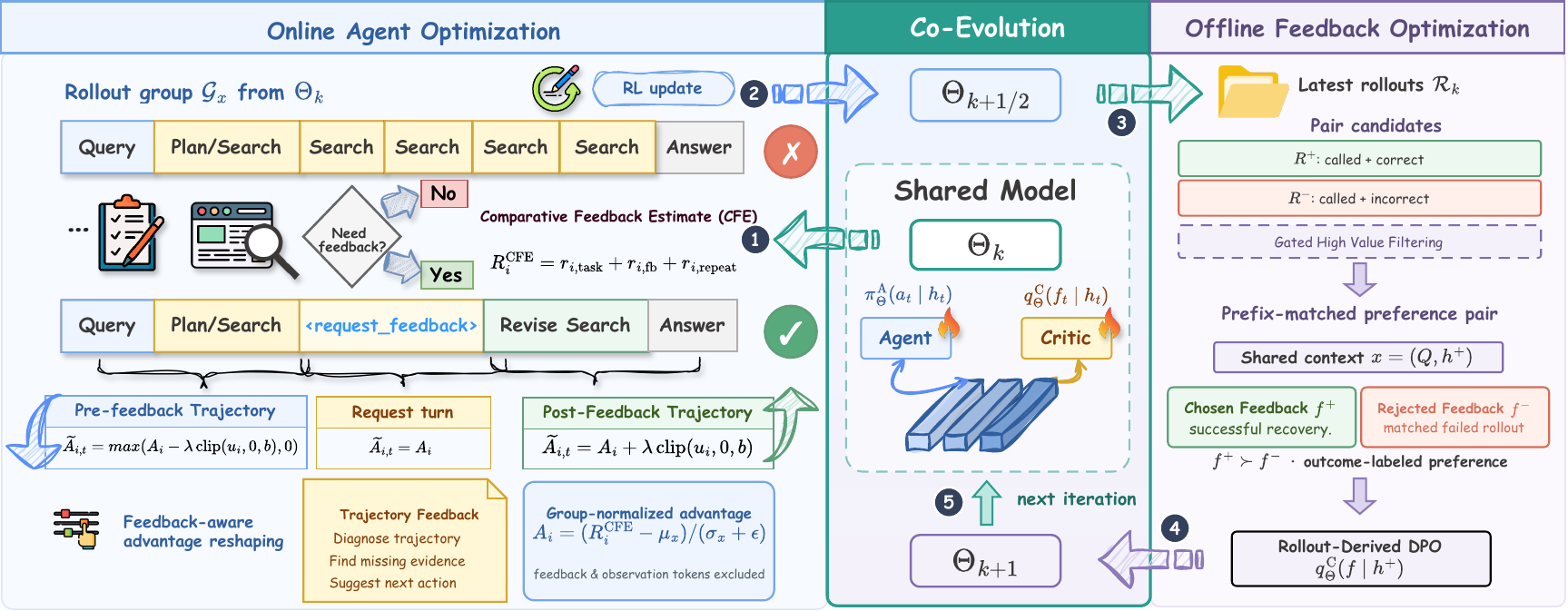}
    \caption{
    Overview of CAFE. One shared model serves as agent and critic. CFE shapes request returns from the call--skip gap; advantage shaping weights tokens before and after feedback. RDPO updates the critic from prefix-matched pairs mined from recent rollouts.
    }
    \label{fig:method_overview}
\end{figure}
\subsection{Self-Feedback Search Agent}
\label{sec:shared_feedback_search}
\noindent\textbf{RQ1.}\ \emph{How should in-trajectory feedback be structured and how can the agent learn to generate it?}

Search agents can make locally plausible choices that nevertheless steer a trajectory away from the evidence and subgoals needed for success. Such early mistakes can compound into repetitive or unproductive actions that become increasingly difficult to reverse \citep{T_3,reseek,wang2026reasoningfailsplanplanningcentric}.

\paragraph{In-Trajectory Feedback Interaction.} We therefore augment the agent's action space with an optional feedback-request action, as illustrated in \Cref{fig:method_overview}. Given a task prompt \(x_i\), the agent interacts with the search environment to produce a rollout \(\tau_i\) and receives a binary outcome reward \(r_i\). At any intermediate history \(h_{i,t}\), the agent may emit a \texttt{<request\_feedback>} action. A critic then conditions on the current trajectory and generates feedback \(f_{i,t}\) that identifies a corrective next step. We append the feedback to the context and return control to the agent, which continues the rollout from this augmented context.


\paragraph{Role-Conditioned Agent–Critic Model.} In-trajectory recovery requires a critic that can diagnose the errors made by the current agent,
while maintaining a separate critic incurs substantial overhead.
Self-rewarding methods motivate the use of a model's own judgments as a learning signal \citep{Self-Rewarding,ssr-zero}. 
Inspired by this, we use the same model to provide corrective feedback during search. We instantiate the agent and critic as role-conditioned behaviors of a single shared model:
\begin{equation}
a_{i,t}\sim\pi_{\theta}^{\mathrm{A}}(\cdot\mid h_{i,t}),
\qquad
f_{i,t}\sim q_{\theta}^{\mathrm{C}}(\cdot\mid h_{i,t}),
\label{eq}
\end{equation}
where \(\mathrm{A}\) and \(\mathrm{C}\) denote the agent and critic roles, respectively. When the agent requests feedback, the model adopts the critic role to generate \(f_{i,t}\), then returns to the agent role and continues from the augmented history. The two roles remain distinct at inference time but share a common backbone.

\paragraph{Bootstrapping Feedback-Conditioned Search.}
To bootstrap both feedback use and feedback generation, we construct SFT data from the base agent's own failure trajectories. We first collect failed rollouts and ask a teacher model (Kimi-K2.5~\citep{kimiteam2026kimik25visualagentic} in our implementation) to identify the earliest turn at which the trajectory becomes erroneous or ceases to make progress. We preserve the agent-generated prefix through this turn, insert a \texttt{<request\_feedback>} action, and use the teacher to generate corrective feedback together with a feedback-conditioned continuation. We retain only repaired trajectories that reach the correct final answer. Unlike pure teacher-generated demonstrations, these trajectories retain the failure patterns encountered by the base agent while providing a successful recovery from them. They therefore serve as high-quality SFT data for initializing the shared model with feedback-augmented search behavior.

\subsection{Online agent optimization for Feedback Seeking and Recovery}
\label{sec:online_agent_optimization}
\noindent\textbf{RQ2.}\ \emph{How can an agent learn when to request feedback and how to utilize it?}

Long-horizon search requires credit signals that are finer grained than a terminal outcome: 
with feedback as an optional intervention, the outcome further reveals neither whether a request was beneficial nor how credit should be assigned around it 
~\citep{IGPO,GiGPO,StepSearch,Information-Self-Locking}. We therefore introduce feedback-aware credit assignment at two levels. Across rollouts, we use the prompt-level empirical call--skip success gap to shape the returns of feedback-requesting trajectories. Within a rollout, we redistribute credit between the behavior that preceded a request and the recovery that followed it.

\paragraph{Comparative Feedback Estimate (CFE) Reward.}
CFE compares rollouts that request feedback with those that skip it. Following GRPO~\citep{shao2024deepseekmathpushinglimitsmathematical}, we sample a rollout group \(\mathcal{G}_x\) for prompt \(x\). Let \(n_{\mathrm{fb},i}\) be the number of requests in rollout \(i\) and \(C_i=\mathbf{1}[n_{\mathrm{fb},i}>0]\). Rollouts with \(C_i=1\) form the call group \(\mathcal{G}_{x,\mathrm{call}}\), while those with \(C_i=0\) form the skip group \(\mathcal{G}_{x,\mathrm{skip}}\). If both groups are nonempty, their empirical success gap is
\begin{equation}
    \widehat{u}(x)
    =
    \frac{1}{|\mathcal{G}_{x,\mathrm{call}}|}
    \sum_{j\in\mathcal{G}_{x,\mathrm{call}}} r_j
    -
    \frac{1}{|\mathcal{G}_{x,\mathrm{skip}}|}
    \sum_{k\in\mathcal{G}_{x,\mathrm{skip}}} r_k .
    \label{eq:observed_call_gap}
\end{equation}
For rollout \(i\), we assign the prompt-level statistic \(u_i=\widehat{u}(x_i)\), and all rollouts for the same prompt receive the same estimate. If either route is absent, we use the batch-level fallback in \Cref{app:cfe_fallback}. We combine this estimate with the task outcome through three terms. The task term \(r_{i,\mathrm{task}}=r_i\) retains the original answer-correctness reward. The feedback term \(r_{i,\mathrm{fb}}=\beta C_i u_i\) applies the group-level gap to rollouts that request feedback. The repeat term \(r_{i,\mathrm{repeat}}=-\gamma[n_{\mathrm{fb},i}-1]_{+}\) penalizes only requests beyond the first, where \([z]_{+}=\max(z,0)\). The scales \(\beta,\gamma\geq0\) control the latter two terms. Their sum gives the CFE-shaped reward:
\begin{equation}
    R_i^{\mathrm{CFE}}
    = r_{i,\mathrm{task}} + r_{i,\mathrm{fb}} + r_{i,\mathrm{repeat}}.
    \label{eq:cfe_reward}
\end{equation}

\paragraph{Feedback-Aware Advantage Shaping.}
Let \(\mu_x\) and \(\sigma_x\) denote the mean and standard deviation of the CFE-shaped returns within \(\mathcal{G}_x\). GRPO assigns every trainable agent token in rollout \(i\) the same normalized advantage:
\begin{equation}
    A_i
    =
    \frac{R_i^{\mathrm{CFE}}-\mu_{x_i}}
    {\sigma_{x_i}+\epsilon},
    \qquad \epsilon>0.
    \label{eq:base_grpo_advantage}
\end{equation}
A single \(A_i\) rewards the prefix that drove the search off course as much as the continuation that repaired it. This pairing is inherited from initialization: because the SFT trajectories in \Cref{sec:shared_feedback_search} retain the base agent's own prefix up to its earliest erroneous turn, a request tends to follow behavior that has already gone wrong. The two parts therefore play opposite roles, and reinforcing them together rewards the very behavior the agent had to abandon.

We accordingly split the trainable agent tokens of each requesting rollout (\(C_i=1\)) at its first request into \(\mathcal{T}_i^{\mathrm{pre}}\), \(\mathcal{T}_i^{\mathrm{call}}\), and \(\mathcal{T}_i^{\mathrm{post}}\), covering the tokens before the request turn, the request turn itself, and the continuation after feedback. Observation and feedback tokens are excluded from the policy loss. We bound the adjustment by clipping the prompt-level gap, \(g_i=\operatorname{clip}(u_i,0,b)\) with \(b>0\), and apply it only when \(A_i>0\) and \(g_i>0\). Eligible rollouts receive:
\begin{equation}
    \widetilde{A}_{i,t}
    =
    \begin{cases}
        \max\!\left(A_i-\lambda g_i,0\right),
        & t\in\mathcal{T}_i^{\mathrm{pre}}, \\[1mm]
        A_i,
        & t\in\mathcal{T}_i^{\mathrm{call}}, \\[1mm]
        A_i+\lambda g_i,
        & t\in\mathcal{T}_i^{\mathrm{post}},
    \end{cases}
    \label{eq:feedback_advantage_shaping}
\end{equation}
where \(\lambda\geq0\) controls the shaping strength. Together the two components answer RQ2: CFE decides across rollouts when a request is worth making, while advantage shaping decides within a rollout which behavior the resulting success should credit.

\subsection{Offline Feedback Optimization and Iterative Co-evolution}
\label{sec:offline_feedback_training}
\noindent\textbf{RQ3.}\ \emph{How can a shared model co-evolve its agent and feedback capabilities?}

Online policy updates continually shift the states and failures the agent encounters, so feedback learned from earlier trajectories may lose relevance. Conversely, improved feedback changes which failures the agent can recover from and thus the rollouts that drive subsequent learning. 
As illustrated in \Cref{fig:method_overview}, CAFE addresses this coupling by alternating online agent optimization with rollout-derived preference optimization (RDPO), which refines the critic from preference pairs mined from the latest on-policy rollouts. Because both roles share parameters, each update changes the data distribution on which the other role is refined.

\paragraph{Outcome-Guided Preference Filtering.}
To optimize the feedback side of this loop, we need to distinguish useful from ineffective guidance. Yet the environment scores only the final answer, which also depends on the preceding search and subsequent actions. We therefore construct outcome-labeled preferences from matched on-policy rollouts. At iteration \(k\), we group the latest rollouts \(\mathcal{R}_k\) by prompt and pair a successful feedback-requesting rollout with an unsuccessful one. After structural filtering, we retain pairs with similar histories at the first request and comparable feedback lengths. Under the successful history as shared context, its feedback \(f^{+}\) is chosen and the feedback \(f^{-}\) from the matched failed rollout is rejected. The retained pairs form \(\mathcal{D}_{k}^{\mathrm{fb}}\). More details can be found in \Cref{app:offline_pairing}.

\paragraph{Offline Feedback Update and Iteration.}
Starting from the online checkpoint \(\theta_{k+\frac12}\), we apply RDPO to \(\mathcal{D}_{k}^{\mathrm{fb}}\), directly preferring feedback associated with successful recovery over matched feedback from unsuccessful trajectories. Because the agent and critic roles share all parameters, RDPO updates the same full-model checkpoint and produces \(\theta_{k+1}\), rather than training a separate critic model. The next online segment then samples fresh trajectories from \(\theta_{k+1}\), from which we rebuild \(\mathcal{D}_{k+1}^{\mathrm{fb}}\) for the subsequent offline update. Through this alternating update, CAFE keeps feedback training aligned with the evolving policy while allowing improved feedback to shape the next round of on-policy experience.
\begin{table}[t!]
\centering

\setlength{\tabcolsep}{4pt}
\renewcommand{\arraystretch}{1.2}
\resizebox{\linewidth}{!}{%
\begin{tabular}{l|cccccccccccccccc}
\Xhline{1.2pt}
\rowcolor{CadetBlue!20}
 & \multicolumn{2}{c}{\textbf{2Wiki}} & \multicolumn{2}{c}{\textbf{HotpotQA}} & \multicolumn{2}{c}{\textbf{MuSiQue}} & \multicolumn{2}{c}{\textbf{PopQA}} & \multicolumn{2}{c}{\textbf{TriviaQA}} & \multicolumn{2}{c}{\textbf{Bamboogle}} & \multicolumn{2}{c}{\textbf{NQ}} & \multicolumn{2}{c}{\textbf{Avg.}} \\
\rowcolor{CadetBlue!20}
\midrule
\textbf{Method} & \textbf{EM} & \textbf{F1} & \textbf{EM} & \textbf{F1} & \textbf{EM} & \textbf{F1} & \textbf{EM} & \textbf{F1} & \textbf{EM} & \textbf{F1} & \textbf{EM} & \textbf{F1} & \textbf{EM} & \textbf{F1} & \textbf{EM} & \textbf{F1} \\
\Xhline{1.2pt}
\multicolumn{17}{c}{\textit{Closed-source Models}} \\
\hline
\rowcolor{gray!10}
GPT-5-Mini~\citeyearpar{singh2026openaigpt5card} & 67.4 & 76.0 & \textbf{59.0} & \textbf{72.0} & 27.2 & 36.5 & 34.6 & 41.1 & 73.0 & 81.2 & \textbf{62.4} & \textbf{73.2} & 31.2 & 42.6 & 50.7 & 60.4\\
\rowcolor{gray!10}
Gemini-2.5-Flash~\citeyearpar{comanici2025gemini25pushingfrontier} & 63.8 & 73.6 & 54.0 & 67.6 & 24.4 & 35.7 & 40.2 & 49.2 & \uline{73.4} & \uline{81.5} & 57.6 & 65.9 & 39.8 & \textbf{53.0} & 50.5 & \textbf{60.9}\\
\rowcolor{gray!10}
Claude-4.5-haiku~\citeyearpar{anthropic2025claudehaiku45systemcard} & 61.2 & 68.0 & 51.0 & 63.0 & 17.0 & 22.2 & 36.8 & 43.3 & 66.8 & 73.6 & 52.8 & 62.4 & 32.8 & 42.9 & 45.5 & 53.6\\
\hline
\multicolumn{17}{c}{\textit{Large Open-source Models}} \\
\hline
\rowcolor{gray!10}
Kimi-K2-thinking~\citeyearpar{kimiteam2026kimik2openagentic} & 67.8 & 76.2 & \uline{57.4} & 69.9 & 29.4 & 37.8 & 33.4 & 38.9 & 72.0 & 78.5 & 50.4 & 58.7 & 30.2 & 41.1 & 48.7 & 57.3\\
\rowcolor{gray!10}
GLM-4.7~\citeyearpar{5team2025glm45agenticreasoningcoding} & 64.4 & 74.6 & \uline{57.4} & \uline{70.9} & 24.8 & 33.6 & 34.4 & 41.2 & \textbf{73.6} & \textbf{81.6} & \uline{58.4} & \uline{69.3} & 33.0 & 45.2 & 49.4 & 59.5\\
\rowcolor{gray!10}
Qwen2.5-72B-Instruct~\citeyearpar{qwen2025qwen25technicalreport} & 61.4 & 69.4 & 52.6 & 65.1 & 26.0 & 35.0 & 35.4 & 43.7 & 66.8 & 74.7 & 57.6 & 66.6 & 35.8 & 45.8 & 47.9 & 57.2\\
\rowcolor{gray!10}
DeepSeek-V4-Flash-preview~\citeyearpar{deepseekai2026deepseekv4highlyefficientmilliontoken} & 59.6 & 66.8 & 51.0 & 62.6 & 17.8 & 23.5 & 22.8 & 26.8 & 61.4 & 67.7 & 47.2 & 52.8 & 28.2 & 38.0 & 41.1 & 48.3\\
\hline
\multicolumn{17}{c}{\textit{RL-based Search Agent Baselines}} \\
\hline
\rowcolor{gray!10}
WebSeer$^{\dagger}$~\citeyearpar{he2026webseer} & 57.8 & 72.3 & 47.8 & 61.3 & 22.6 & 35.1 & 31.6 & 40.8 & 57.8 & 67.3 & 47.2 & 61.2 & 21.6 & 31.9 & 40.9 & 52.8\\
\rowcolor{gray!10}
Search-R1$^{*}$ ~\citeyearpar{jin2025searchr1trainingllmsreason} & 67.0 & 75.4 & 48.4 & 60.9 & 25.8 & 36.2 & 41.0 & 46.9 & 65.0 & 70.8 & 47.2 & 58.4 & 39.8 & 49.1 & 47.7 & 56.8\\
\rowcolor{gray!10}
R-Search$^{*}$ ~\citeyearpar{R-search} & 69.8 & 77.7 & 52.2 & 64.4 & \uline{31.4} & \textbf{41.6} & 41.8 & 48.1 & 64.2 & 71.7 & 42.4 & 57.6 & 38.0 & 49.1 & 48.5 & 58.6\\
\rowcolor{gray!10}
IGPO$^{\ddagger}$ ~\citeyearpar{IGPO} & 79.6 & 86.1 & 53.4 & 65.1 & 27.8 & 36.9 & 43.2 & 49.1 & 63.2 & 71.4 & 48.8 & 59.1 & 36.8 & 48.1 & 50.4 & 59.4\\
\rowcolor{gray!10}
StepSearch$^{\dagger}$~\citeyearpar{StepSearch} & 52.6 & 63.2 & 45.2 & 54.5 & 29.2 & 38.8 & 32.2 & 39.1 & 53.2 & 61.7 & 40.0 & 51.2 & 33.6 & 44.1 & 40.9 & 50.4\\
\hline
\multicolumn{17}{c}{\textit{CAFE}} \\
\hline
\rowcolor{gray!10}
Qwen2.5-7B-Instruct~\citeyearpar{qwen2025qwen25technicalreport} & 44.8 & 54.4 & 41.8 & 53.6 & 20.4 & 28.8 & 34.4 & 42.1 & 56.4 & 65.6 & 37.6 & 49.3 & 31.6 & 41.3 & 38.1 & 47.9\\
\rowcolor{gray!10}
\hspace{1em}+ SFT & 63.1 & 72.4 & 44.2 & 55.0 & 19.0 & 27.5 & 35.2 & 42.1 & 52.6 & 62.5 & 42.4 & 51.8 & 28.8 & 39.1 & 40.8 & 50.1\\
\rowcolor{gray!10}
\hspace{1em}+ SFT + GRPO & 80.6 & 86.6 & 50.8 & 61.0 & 27.2 & 36.6 & 44.6 & 49.0 & 60.4 & 67.9 & 45.6 & 56.2 & 38.4 & 48.6 & 49.7 & 58.0\\
\hline
\rowcolor{gray!10}
\textbf{\hspace{1em}+ SFT + CAFE} & \textbf{84.0} & \textbf{89.2} & 53.4 & 64.8 & 30.2 & 39.1 & \textbf{46.4} & \textbf{51.6} & 62.6 & 69.9 & 50.4 & 61.4 & \textbf{40.8} & 49.2 & \textbf{52.5} & \uline{60.7}\\
\Xhline{1.2pt}
\end{tabular}%
}
\caption{Main results on seven agentic SearchQA benchmarks. Best and second-best results are \textbf{bolded} and \uline{underlined}. \(\dagger\) Released checkpoint evaluated under our protocol. \(*\) Results reported in the original paper. \(\ddagger\) Our reproduction initialized from our SFT checkpoint.}
\label{tab:main_results}
\end{table}
\begin{table}[t]
\centering
\setlength{\tabcolsep}{4pt}
\renewcommand{\arraystretch}{1.2}
\resizebox{\linewidth}{!}{%
\begin{tabular}{cc|cccccccccccccccc}
\Xhline{1.2pt}
\rowcolor{CadetBlue!20}
\multicolumn{2}{c|}{\textbf{Online RL}}
& \multicolumn{2}{c}{\textbf{2Wiki}}
& \multicolumn{2}{c}{\textbf{HotpotQA}}
& \multicolumn{2}{c}{\textbf{MuSiQue}}
& \multicolumn{2}{c}{\textbf{PopQA}}
& \multicolumn{2}{c}{\textbf{TriviaQA}}
& \multicolumn{2}{c}{\textbf{Bamboogle}}
& \multicolumn{2}{c}{\textbf{NQ}}
& \multicolumn{2}{c}{\textbf{Avg.}} \\
\rowcolor{CadetBlue!20}
\midrule
\textbf{CFE} & \textbf{Adv.\ Shaping}
& \textbf{EM} & \textbf{F1}
& \textbf{EM} & \textbf{F1}
& \textbf{EM} & \textbf{F1}
& \textbf{EM} & \textbf{F1}
& \textbf{EM} & \textbf{F1}
& \textbf{EM} & \textbf{F1}
& \textbf{EM} & \textbf{F1}
& \textbf{EM} & \textbf{F1} \\
\Xhline{1.2pt}
\rowcolor{gray!10}
\xmark & \xmark & 80.6 & 86.6 & 50.8 & 61.0 & 27.2 & 36.6 & 44.6 & \uline{49.0} & 60.4 & 67.9 & 45.6 & 56.2 & 38.4 & \uline{48.6} & 49.7 & 58.0\\
\rowcolor{gray!10}
\cmark & \xmark & \uline{83.2} & 88.0 & \uline{52.4} & 63.0 & 28.4 & 36.5 & \uline{45.0} & 48.9 & 60.8 & 68.5 & 47.2 & 56.6 & \uline{38.7} & 47.4 & 50.8 & 58.4\\
\rowcolor{gray!10}
\xmark & \cmark & \uline{83.2} & \uline{88.2} & \uline{52.4} & \uline{63.2} & \uline{30.6} & \uline{38.5} & 44.6 & 48.9 & \textbf{61.8} & \textbf{69.6} & \uline{47.2} & \uline{58.9} & 38.6 & 48.2 & \uline{51.2} & \uline{59.4}\\
\rowcolor{gray!10}
\cmark & \cmark & \textbf{83.4} & \textbf{88.3} & \textbf{53.0} & \textbf{63.9} & \textbf{31.8} & \textbf{40.5} & \textbf{46.2} & \textbf{50.8} & \uline{61.2} & \uline{69.2} & \textbf{48.0} & \textbf{60.5} & \textbf{40.0} & \textbf{48.8} & \textbf{51.9} & \textbf{60.3}\\
\Xhline{1.2pt}
\end{tabular}%
}
\caption{Online optimization ablation across seven benchmarks. Best and second-best results are \textbf{bolded} and \uline{underlined}. \textbf{CFE} augments the task reward with the prompt-level call--skip success gap and a repeated-request cost; \textbf{Adv.\ Shaping} reweights pre- and post-feedback token advantages.}
\label{tab:online_component_ablation}
\end{table}
\section{Experiments}
\paragraph{Dataset and Metrics.} We evaluate our method on seven agentic SearchQA benchmarks. We use 2WikiMultihopQA~\citep{2wiki} as the sole in-domain benchmark and assess out-of-domain generalization on HotpotQA~\citep{yang-etal-2018-hotpotqa}, MuSiQue~\citep{musique}, PopQA~\citep{popqa}, Bamboogle~\citep{bamboogle}, Natural Questions~\citep{nq}, and TriviaQA~\citep{joshi-etal-2017-triviaqa}. We report Exact Match (EM) and token-level F1 scores. The dataset characteristics, versions, and evaluation sizes are provided in Appendix~\ref{app:dataset_details}. 

\paragraph{Baselines and Implementation.} We compare against three groups of baselines. \textbf{(i) Closed-source LLMs.} GPT-5-Mini~\citep{singh2026openaigpt5card}, Gemini-2.5-Flash~\citep{comanici2025gemini25pushingfrontier}, and Claude-4.5-Haiku~\citep{anthropic2025claudehaiku45systemcard} are prompted as search agents without task-specific training. \textbf{(ii) Large open-source LLMs.} We also include Kimi-K2-Thinking~\citep{kimiteam2026kimik2openagentic}, GLM-4.7~\citep{5team2025glm45agenticreasoningcoding}, DeepSeek-V4-Flash~\citep{deepseekai2026deepseekv4highlyefficientmilliontoken}, and Qwen2.5-72B-Instruct~\citep{qwen2025qwen25technicalreport}. We evaluate all models in groups (i)--(ii) under our protocol. \textbf{(iii) RL-based search agents.} We evaluate the released checkpoints of WebSeer-14B\(^{\dagger}\)~\citep{he2026webseer} and StepSearch\(^{\dagger}\)~\citep{StepSearch} on our evaluation sets. Search-R1\(^{*}\)~\citep{jin2025searchr1trainingllmsreason} and R-Search\(^{*}\)~\citep{R-search} use the results reported in the original papers. IGPO\(^{\ddagger}\)~\citep{IGPO} is reproduced by applying its RL procedure to our feedback-based SFT checkpoint. We use Qwen2.5-7B-Instruct~\citep{qwen2025qwen25technicalreport} as the shared backbone for both CAFE roles. Full implementation details are provided in \Cref{sec:implementation_detail}.

\paragraph{Main Results.}
\Cref{tab:main_results} shows three main findings. (1) With a 7B backbone, CAFE achieves the highest average EM (52.5) and the second-highest average F1 (60.7) among all evaluated methods, outperforming the strongest RL-based baseline IGPO by 2.1 EM and 1.3 F1. (2) The lower block shows steady gains across training stages. Feedback-augmented SFT improves the backbone from 38.1/47.9 to 40.8/50.1 in average EM/F1, while GRPO raises the scores to 49.7/58.0. CAFE further reaches 52.5/60.7, confirming that iterative feedback optimization provides gains beyond a stronger search policy alone. (3) Compared with GRPO, CAFE improves both metrics on every benchmark, including all six out-of-domain datasets. Relative to Search-R1, its average gains are 7.4 EM and 5.9 F1 across the four multi-hop benchmarks, compared with 1.3 EM and 1.3 F1 across the three single-hop benchmarks. This gap is consistent with our motivation: in-context feedback can correct an intermediate search error before it affects the remaining retrieval steps.

At 3B scale, CAFE again improves substantially over the initial checkpoint, reaching performance comparable to several 7B search baselines (full results in \Cref{tab:model_size_ablation}). To test CAFE in a more challenging long-horizon deep-research setting, we evaluate the 7B checkpoints on BrowseComp-Plus~\citep{chen2025browsecompplusfairtransparentevaluation}. Performance improves at each training stage, with CAFE achieving the best result (\Cref{tab:browsecomp_plus_results}), extending the same trend beyond standard SearchQA benchmarks.

\section{Ablation And Analysis}
\label{sec:analysis_ablation}

\subsection{Component Ablations}
\label{sec:component_ablation}

\begin{table}[t]
\centering
\setlength{\tabcolsep}{4pt}
\renewcommand{\arraystretch}{1.2}
\resizebox{\linewidth}{!}{%
\begin{tabular}{cc|cccccccccccccccc}
\Xhline{1.2pt}
\rowcolor{CadetBlue!20}
\multicolumn{2}{c|}{\textbf{Online RL}}
& \multicolumn{2}{c}{\textbf{2Wiki}}
& \multicolumn{2}{c}{\textbf{HotpotQA}}
& \multicolumn{2}{c}{\textbf{MuSiQue}}
& \multicolumn{2}{c}{\textbf{PopQA}}
& \multicolumn{2}{c}{\textbf{TriviaQA}}
& \multicolumn{2}{c}{\textbf{Bamboogle}}
& \multicolumn{2}{c}{\textbf{NQ}}
& \multicolumn{2}{c}{\textbf{Avg.}} \\
\rowcolor{CadetBlue!20}
\midrule
\textbf{CFE} & \textbf{Adv.\ Shaping}
& \textbf{EM} & \textbf{F1}
& \textbf{EM} & \textbf{F1}
& \textbf{EM} & \textbf{F1}
& \textbf{EM} & \textbf{F1}
& \textbf{EM} & \textbf{F1}
& \textbf{EM} & \textbf{F1}
& \textbf{EM} & \textbf{F1}
& \textbf{EM} & \textbf{F1} \\
\Xhline{1.2pt}
\multicolumn{18}{c}{\textit{Rollout-Derived SFT (RSFT)}} \\
\hline
\rowcolor{gray!10}
\xmark & \xmark & 80.0 & 85.1 & \textbf{54.2} & \uline{64.0} & 26.8 & 35.9 & \uline{46.0} & \textbf{51.8} & 60.8 & 68.3 & \textbf{50.4} & \uline{59.7} & 36.6 & 45.9 & 50.7 & 58.7\\
\rowcolor{gray!10}
\cmark & \xmark & 80.6 & 85.2 & 51.4 & 62.0 & 26.4 & 35.4 & 43.2 & 47.5 & 61.0 & 67.8 & 44.8 & 56.3 & 36.8 & 47.3 & 49.2 & 57.4\\
\rowcolor{gray!10}
\xmark & \cmark & 79.8 & 85.7 & 51.0 & 62.2 & 26.4 & 35.0 & 44.2 & 49.0 & 62.6 & 69.9 & \uline{49.6} & 58.8 & \uline{40.6} & \uline{49.3} & 50.6 & 58.6\\
\rowcolor{gray!10}
\cmark & \cmark & 82.4 & 87.2 & 52.4 & 63.1 & 27.2 & 36.2 & \uline{46.0} & 49.8 & \uline{63.2} & \uline{70.0} & 48.0 & 56.2 & 39.0 & 49.1 & 51.2 & 58.8\\
\hline
\multicolumn{18}{c}{\textit{Rollout-Derived DPO (RDPO)}} \\
\hline
\rowcolor{gray!10}
\xmark & \xmark & 81.4 & 86.9 & 51.8 & 62.9 & 27.4 & 36.4 & 45.4 & 50.2 & 61.2 & 68.4 & 47.2 & 56.3 & 37.6 & 48.3 & 50.3 & 58.5\\
\rowcolor{gray!10}
\cmark & \xmark & \uline{83.2} & \uline{87.7} & 52.6 & 63.4 & 27.6 & 37.2 & 45.6 & 50.2 & 61.4 & 68.4 & \uline{49.6} & 59.2 & 39.2 & \textbf{49.7} & \uline{51.3} & 59.4\\
\rowcolor{gray!10}
\xmark & \cmark & 83.0 & \uline{87.7} & 52.6 & 63.8 & \uline{28.0} & \uline{37.6} & 45.6 & 50.3 & \textbf{64.0} & \textbf{71.5} & 47.2 & 58.2 & 37.2 & 47.2 & 51.1 & \uline{59.5}\\
\rowcolor{gray!10}
\cmark & \cmark & \textbf{84.0} & \textbf{89.2} & \uline{53.4} & \textbf{64.8} & \textbf{30.2} & \textbf{39.1} & \textbf{46.4} & \uline{51.6} & 62.6 & 69.9 & \textbf{50.4} & \textbf{61.4} & \textbf{40.8} & 49.2 & \textbf{52.5} & \textbf{60.7}\\
\Xhline{1.2pt}
\end{tabular}%
}
\caption{Offline optimization ablation under a matched five-iteration schedule with 100 online RL steps per iteration. Best and second-best results are \textbf{bolded} and \uline{underlined}. Row groups specify the offline objective (RSFT or RDPO), while the columns indicate the online RL components used to train the starting checkpoint.}
\label{tab:offline_component_ablation}
\end{table}

\begin{figure}[t]
\centering
\begin{subfigure}[b]{0.385\linewidth}
  \centering
  \includegraphics[height=1.55in]{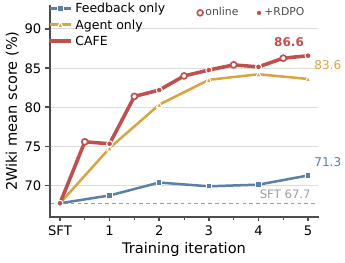}
  \caption{Agent--feedback co-evolution}
  \label{fig:coevolution_protocols}
\end{subfigure}%
\hfill%
\begin{subfigure}[b]{0.305\linewidth}
  \centering
  \includegraphics[height=1.55in]{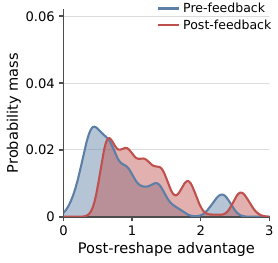}
  \caption{Early training}
  \label{fig:adv_early}
\end{subfigure}%
\hfill%
\begin{subfigure}[b]{0.305\linewidth}
  \centering
  \includegraphics[height=1.55in]{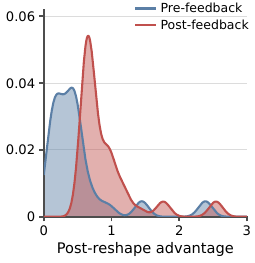}
  \caption{Late training}
  \label{fig:adv_late}
\end{subfigure}
\caption{Analysis of co-evolution and feedback-aware credit assignment. In panel (a), agent-only optimization pairs the evolving agent with the frozen SFT feedback model, whereas feedback-only optimization pairs the evolving feedback model with the frozen SFT agent. For CAFE, open markers denote intermediate online-RL checkpoints and filled markers denote the RDPO checkpoints.}
\label{fig:analysis_panels}
\end{figure}

\paragraph{Online Optimization.}
\Cref{tab:online_component_ablation} compares CFE and advantage shaping under the same SFT initialization and RL budget. CFE alone raises the average EM/F1 from 49.7/58.0 to 50.8/58.4. The gain is modest but consistent with its purpose: CFE changes the return associated with requesting feedback. Advantage shaping has a larger effect, reaching 51.2/59.4 by differentially weighting agent tokens before and after feedback. This difference is also reflected in \Cref{fig:adv_early,fig:adv_late}. The two segments begin with similar advantage distributions, but later in training the pre-feedback advantages concentrate near zero while the post-feedback advantages shift toward a positive mode. Using both components gives the best result at 51.9 EM and 60.3 F1, with improvements on every benchmark. Full test-accuracy and empirical routing-entropy trajectories over 500 steps are reported in \Cref{app:training_dynamics,fig:training_dynamics}. CAFE achieves the highest late-stage test accuracy while maintaining higher feedback-policy entropy. The largest EM gain occurs on MuSiQue, while the largest F1 gain occurs on Bamboogle. Both are multi-hop benchmarks where correcting an intermediate error can affect several subsequent retrieval steps.

\paragraph{Offline Optimization.}
Learning feedback from rollout outcomes is not straightforward because the terminal label applies to the entire trajectory rather than to the feedback itself.
We compare RDPO with rollout-derived SFT (RSFT), which uses the same mining pipeline as RDPO but retains only feedback from successful rollouts. This positive-only objective is noisy: a trajectory may succeed because of its search prefix or subsequent actions even when the feedback is uninformative. RDPO instead preserves the comparison with a failed rollout matched by prompt and pre-feedback state, providing a cleaner signal for feedback quality. The relative objective is also better suited to the shared model, since it remains anchored to the online checkpoint, whereas maximum-likelihood fitting can shift both roles without such a constraint. RDPO outperforms RSFT in three of four settings (\Cref{tab:offline_component_ablation}). The conducted training schedule comparison in \Cref{app:iteration_schedule} further identifies \(100\times5\) as the strongest alternation schedule, motivating our default.

\subsection{Co-evolution Analysis}
\label{sec:coevolution_analysis}
We compare five rounds of feedback-only, agent-only, and alternating optimization from the same feedback-augmented SFT checkpoint. In the agent-only control, online RL updates the agent while every feedback request is answered by a frozen copy of the SFT model. In the feedback-only control, the SFT agent remains fixed while RDPO updates the model that generates its requested feedback. CAFE instead alternates online RL and RDPO, with the latest shared checkpoint serving both roles. Performance is measured on 2Wiki using the mean of EM and F1. As shown in \Cref{fig:coevolution_protocols}, feedback-only optimization improves the score from 67.7 to 71.3, while agent-only optimization peaks at 84.2 before ending at 83.6. Alternating optimization reaches 86.6, outperforming the final agent-only checkpoint by 3.0 points. The one-sided controls show that improving either capability helps, while updating both allows the gains to continue across iterations.

We further test whether these gains reflect stage-specific alignment rather than a uniformly stronger critic by cross-playing agent and critic checkpoints across iterations. As shown in \Cref{fig:coevolution_crossplay_heatmaps}, from iterations 3 through 5, each agent performs best with the critic from the same iteration. Holding the final agent fixed and replacing the SFT critic with the iteration-5 critic raises EM from \(80.6\) to \(84.0\) and F1 from \(86.6\) to \(89.2\). Conversely, the iteration-5 critic is not universally best for earlier agents, indicating that the gains arise from alignment with the policy's evolving failure distribution rather than critic strength alone.

\begin{figure}[t]
\centering

\includegraphics[width=0.92\linewidth]{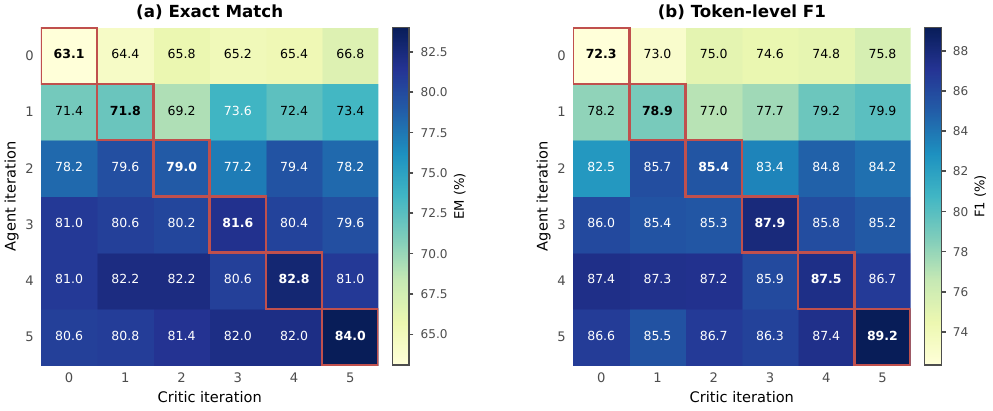}
\caption{Agent--critic cross-play on 2Wiki. Rows and columns denote agent and
critic training iterations, respectively. Iteration 0 is the shared SFT
initialization. Cells report EM or token-level F1, and red boxes mark
same-iteration pairs.}
\label{fig:coevolution_crossplay_heatmaps}
\end{figure}

\subsection{Feedback Evolution Analysis}
\label{sec:feedback_evolution_analysis}
To track how feedback evolves during training, \Cref{fig:feedback_word_trend} visualizes frequent terms from trajectories collected after iterations 1, 3, and 5, representing the early, middle, and late stages. Early feedback is dominated by retrieval and grounding errors, including \textbf{\emph{misread results}} and \textbf{\emph{conflated entities}}. In the middle stage, the emphasis shifts toward careful evidence verification, reflected by terms such as \textbf{\emph{valid answers}}, and \textbf{\emph{explicitly stated}}. Late feedback increasingly targets residual search and reasoning inefficiencies, including \textbf{\emph{repeated query}}, \textbf{\emph{redundant tool calls}}, and \textbf{\emph{logic fails}}. This progression indicates that as basic retrieval and grounding failures recede, the critic adapts to the agent's evolving error profile by focusing increasingly on higher-level planning and execution errors.

\subsection{Hallucination Analysis}
\label{sec:hallucination_analysis}
Long-horizon search requires an agent to integrate evidence across many retrieval and reasoning steps, making the final answer vulnerable to unsupported claims carried forward from earlier errors. We therefore evaluate answer-level hallucination, marking an answer as hallucinated if it contains at least one factual claim unsupported by the evidence retrieved along its search trajectory. The base model produces an average hallucination rate of 29.9\%, which drops to 17.6\% after outcome-reward GRPO and further to 12.6\% with CAFE. As shown in \Cref{fig:hallucination_rate}, CAFE reduces hallucinations relative to GRPO on every benchmark, with the largest reductions on NQ (10.8 percentage points) and MuSiQue (9.4 points).

\begin{figure}[t]

\centering
\begin{subfigure}[b]{0.56\linewidth}
  \centering
  \includegraphics[height=1.30in]{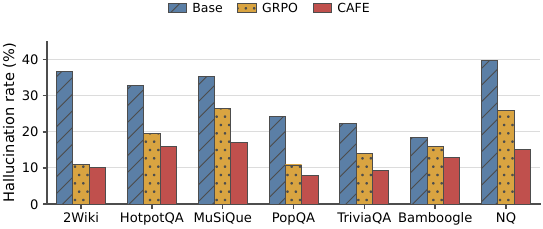}
  \caption{Hallucination rate}
  \label{fig:hallucination_rate}
\end{subfigure}%
\hfill%
\begin{subfigure}[b]{0.42\linewidth}
  \centering
  \includegraphics[height=1.50in]{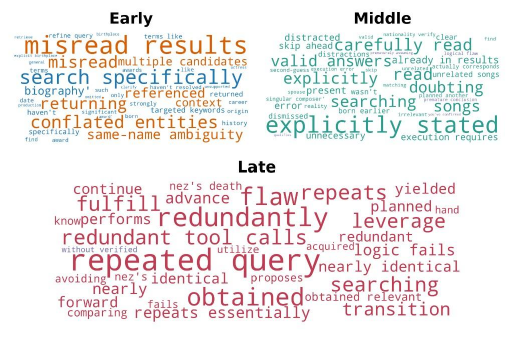}
  \caption{Feedback-content evolution}
  \label{fig:feedback_word_trend}
\end{subfigure}
\caption{Analysis of answer grounding and feedback evolution. (a) Hallucination rates across seven benchmarks, where lower is better. (b) Dominant feedback terms during early, middle, and late training.}

\label{fig:feedback_quality_analysis}
\end{figure}

\section{Related Work}
\textbf{Search Agent and Agentic Credit Assignment.}
Search agents have progressed from pipelines that interleave reasoning and
retrieval \citep{trivedi2023interleaving,ReAct} to outcome-supervised policies
that learn when and how to search
\citep{jin2025searchr1trainingllmsreason,song2025r1searcherincentivizingsearchcapability}.
Their long trajectories nevertheless retain a sparse-credit problem: terminal
correctness does not identify which intermediate decisions were useful. Recent
methods provide finer signals through information gain and confidence changes
\citep{StepSearch,IGPO,TIPS}, local state comparisons and advantage shaping
\citep{GiGPO,DeepPlanner}, diagnostic or directional signals
\citep{CriticSearch,T_3,Information-Self-Locking}, and process rewards, search
hints, or multi-agent refinement \citep{InfoFlow}. These approaches sharpen
supervision for intermediate search behavior. Once explicit feedback
intervenes, however, success also couples the behavior that prompted the
request with the recovery that followed it, creating a distinct
feedback-conditioned credit boundary.

\textbf{Self-Reflection and Corrective Feedback.}
Natural-language correction has been explored through prompted inference-time
loops \citep{reflexion,Self-Refine,CRITIC} and through trained
self-verification, self-correction, or critique models
\citep{ScoRe,S2R,CTRL}. In search settings, ReSeek equips trajectories with
evidence judgments and replanning \citep{reseek}, whereas WebSeer uses
answer-submission-triggered outcome feedback to support continued search
\citep{he2026webseer}. More closely related, ECHO co-evolves separate policy
and critic models through score-aware hindsight refinement, but its critic is
invoked only after a trajectory is completed and therefore cannot redirect the
ongoing search before errors compound \citep{li-etal-2026-stale}. CAFE instead
treats feedback as an optional intervention within the active trajectory.

\textbf{Self-Evolving Agents.}
Self-evolving agents use generated experience to update task policies,
supervisory signals, or system components. EvolveSearch alternates supervised
fine-tuning on filtered trajectories with reinforcement learning exploration
\citep{evolvesearch}, while Self-Rewarding and EvoLM jointly improve task
policies and supervisory signals \citep{Self-Rewarding,evolm}. Retroformer
updates reflection and prompt revision from environmental feedback
\citep{yao2024retroformer}; AFlow, G\"odel Agent, and ADAS extend optimization
to agent designs, workflows, and runtime logic
\citep{zhang2025aflow,DBLP:conf/acl/YinWPL0W25,ADAS}. These lines of work
optimize different parts of the improvement loop. We study their coupling
across two timescales: feedback is requested and used within a trajectory,
while recent rollout outcomes update feedback generation across iterations,
changing the experience available to both roles in the next round.
\section{Conclusion}
We introduced CAFE, a shared-model framework that jointly adapts an agent's ability to use feedback and its ability to generate it. Online RL trains the agent to request and act on feedback, while rollout-derived preference optimization updates the critic using recent trajectories. Across seven search QA benchmarks, CAFE achieves the strongest average performance among the evaluated RL-based agents, transfers to six out-of-domain datasets, and reduces answer-level hallucinations. The broader lesson is that acting and critiquing form a coupled learning system: each changes the experience from which the other improves. A self-improving agent therefore needs feedback that evolves with its policy, rather than a fixed supervisor tied to an earlier distribution of failures.
\clearpage
\bibliography{iclr2026_conference}

@inproceedings{R-search,
  author       = {Qingfei Zhao and
                  Ruobing Wang and
                  Dingling Xu and
                  Daren Zha and
                  Bowen Ma and
                  Zhichun Wang and
                  Shijie Jia and
                  Limin Liu and
                  Xin Wang},
  editor       = {Maria Liakata and
                  Viviane P. Moreira and
                  Jiajun Zhang and
                  David Jurgens},
  title        = {R-Search: Empowering {LLM} Reasoning with Search via Multi-Reward
                  Reinforcement Learning},
  booktitle    = {Findings of the Association for Computational Linguistics, {ACL} 2026,
                  San Diego, California, United States, July 2-7, 2026},
  pages        = {38030--38046},
  publisher    = {Association for Computational Linguistics},
  year         = {2026},
  url          = {https://aclanthology.org/2026.findings-acl.1896/},
  bibsource    = {dblp computer science bibliography, https://dblp.org}
}

@misc{jin2025searchr1trainingllmsreason,
      title={Search-R1: Training LLMs to Reason and Leverage Search Engines with Reinforcement Learning}, 
      author={Bowen Jin and Hansi Zeng and Zhenrui Yue and Jinsung Yoon and Sercan Arik and Dong Wang and Hamed Zamani and Jiawei Han},
      year={2025},
      eprint={2503.09516},
      archivePrefix={arXiv},
      primaryClass={cs.CL},
      url={https://arxiv.org/abs/2503.09516}, 
}

@misc{song2025r1searcherincentivizingsearchcapability,
      title={R1-Searcher: Incentivizing the Search Capability in LLMs via Reinforcement Learning}, 
      author={Huatong Song and Jinhao Jiang and Yingqian Min and Jie Chen and Zhipeng Chen and Wayne Xin Zhao and Lei Fang and Ji-Rong Wen},
      year={2025},
      eprint={2503.05592},
      archivePrefix={arXiv},
      primaryClass={cs.AI},
      url={https://arxiv.org/abs/2503.05592}, 
}

@inproceedings{2wiki,
  author       = {Xanh Ho and
                  Anh{-}Khoa Duong Nguyen and
                  Saku Sugawara and
                  Akiko Aizawa},
  editor       = {Donia Scott and
                  N{\'{u}}ria Bel and
                  Chengqing Zong},
  title        = {Constructing {A} Multi-hop {QA} Dataset for Comprehensive Evaluation
                  of Reasoning Steps},
  booktitle    = {Proceedings of the 28th International Conference on Computational
                  Linguistics, {COLING} 2020, Barcelona, Spain (Online), December 8-13,
                  2020},
  pages        = {6609--6625},
  publisher    = {International Committee on Computational Linguistics},
  year         = {2020},
  url          = {https://doi.org/10.18653/v1/2020.coling-main.580},
  doi          = {10.18653/V1/2020.COLING-MAIN.580},
  bibsource    = {dblp computer science bibliography, https://dblp.org}
}

@inproceedings{yang-etal-2018-hotpotqa,
    title = "{H}otpot{QA}: A Dataset for Diverse, Explainable Multi-hop Question Answering",
    author = "Yang, Zhilin  and
      Qi, Peng  and
      Zhang, Saizheng  and
      Bengio, Yoshua  and
      Cohen, William  and
      Salakhutdinov, Ruslan  and
      Manning, Christopher D.",
    editor = "Riloff, Ellen  and
      Chiang, David  and
      Hockenmaier, Julia  and
      Tsujii, Jun{'}ichi",
    booktitle = "Proceedings of the 2018 Conference on Empirical Methods in Natural Language Processing",
    month = oct # "-" # nov,
    year = "2018",
    address = "Brussels, Belgium",
    publisher = "Association for Computational Linguistics",
    url = "https://aclanthology.org/D18-1259/",
    doi = "10.18653/v1/D18-1259",
    pages = "2369--2380"
}

@article{musique,
  author       = {Harsh Trivedi and
                  Niranjan Balasubramanian and
                  Tushar Khot and
                  Ashish Sabharwal},
  title        = {{MuSiQue}: Multihop Questions via Single-hop Question
                  Composition},
  journal      = {Trans. Assoc. Comput. Linguistics},
  volume       = {10},
  pages        = {539--554},
  year         = {2022},
  url          = {https://doi.org/10.1162/tacl\_a\_00475},
  doi          = {10.1162/TACL\_A\_00475},
  bibsource    = {dblp computer science bibliography, https://dblp.org}
}

@inproceedings{popqa,
    title = "When Not to Trust Language Models: Investigating Effectiveness of Parametric and Non-Parametric Memories",
    author = "Mallen, Alex  and
      Asai, Akari  and
      Zhong, Victor  and
      Das, Rajarshi  and
      Khashabi, Daniel  and
      Hajishirzi, Hannaneh",
    editor = "Rogers, Anna  and
      Boyd-Graber, Jordan  and
      Okazaki, Naoaki",
    booktitle = "Proceedings of the 61st Annual Meeting of the Association for Computational Linguistics (Volume 1: Long Papers)",
    month = jul,
    year = "2023",
    address = "Toronto, Canada",
    publisher = "Association for Computational Linguistics",
    url = "https://aclanthology.org/2023.acl-long.546/",
    doi = "10.18653/v1/2023.acl-long.546",
    pages = "9802--9822"
}

@inproceedings{bamboogle,
  author       = {Ofir Press and
                  Muru Zhang and
                  Sewon Min and
                  Ludwig Schmidt and
                  Noah A. Smith and
                  Mike Lewis},
  editor       = {Houda Bouamor and
                  Juan Pino and
                  Kalika Bali},
  title        = {Measuring and Narrowing the Compositionality Gap in Language Models},
  booktitle    = {Findings of the Association for Computational Linguistics: {EMNLP}
                  2023, Singapore, December 6-10, 2023},
  series       = {Findings of {ACL}},
  volume       = {{EMNLP} 2023},
  pages        = {5687--5711},
  publisher    = {Association for Computational Linguistics},
  year         = {2023},
  url          = {https://doi.org/10.18653/v1/2023.findings-emnlp.378},
  doi          = {10.18653/V1/2023.FINDINGS-EMNLP.378},
  bibsource    = {dblp computer science bibliography, https://dblp.org}
}

@article{nq,
    title = "Natural Questions: A Benchmark for Question Answering Research",
    author = "Kwiatkowski, Tom  and
      Palomaki, Jennimaria  and
      Redfield, Olivia  and
      Collins, Michael  and
      Parikh, Ankur  and
      Alberti, Chris  and
      Epstein, Danielle  and
      Polosukhin, Illia  and
      Devlin, Jacob  and
      Lee, Kenton  and
      Toutanova, Kristina  and
      Jones, Llion  and
      Kelcey, Matthew  and
      Chang, Ming-Wei  and
      Dai, Andrew M.  and
      Uszkoreit, Jakob  and
      Le, Quoc  and
      Petrov, Slav",
    editor = "Lee, Lillian  and
      Johnson, Mark  and
      Roark, Brian  and
      Nenkova, Ani",
    journal = "Transactions of the Association for Computational Linguistics",
    volume = "7",
    year = "2019",
    address = "Cambridge, MA",
    publisher = "MIT Press",
    url = "https://aclanthology.org/Q19-1026/",
    doi = "10.1162/tacl_a_00276",
    pages = "452--466"
}

@inproceedings{joshi-etal-2017-triviaqa,
    title = "{T}rivia{QA}: A Large Scale Distantly Supervised Challenge Dataset for Reading Comprehension",
    author = "Joshi, Mandar  and
      Choi, Eunsol  and
      Weld, Daniel  and
      Zettlemoyer, Luke",
    editor = "Barzilay, Regina  and
      Kan, Min-Yen",
    booktitle = "Proceedings of the 55th Annual Meeting of the Association for Computational Linguistics (Volume 1: Long Papers)",
    month = jul,
    year = "2017",
    address = "Vancouver, Canada",
    publisher = "Association for Computational Linguistics",
    url = "https://aclanthology.org/P17-1147/",
    doi = "10.18653/v1/P17-1147",
    pages = "1601--1611"
}

@inproceedings{reflexion,
  author       = {Noah Shinn and
                  Federico Cassano and
                  Ashwin Gopinath and
                  Karthik Narasimhan and
                  Shunyu Yao},
  editor       = {Alice Oh and
                  Tristan Naumann and
                  Amir Globerson and
                  Kate Saenko and
                  Moritz Hardt and
                  Sergey Levine},
  title        = {Reflexion: language agents with verbal reinforcement learning},
  booktitle    = {Advances in Neural Information Processing Systems 36: Annual Conference
                  on Neural Information Processing Systems 2023, NeurIPS 2023, New Orleans,
                  LA, USA, December 10 - 16, 2023},
  year         = {2023},
  url          = {http://papers.nips.cc/paper\_files/paper/2023/hash/1b44b878bb782e6954cd888628510e90-Abstract-Conference.html},
  bibsource    = {dblp computer science bibliography, https://dblp.org}
}

@inproceedings{Self-Refine,
 author = {Madaan, Aman and Tandon, Niket and Gupta, Prakhar and Hallinan, Skyler and Gao, Luyu and Wiegreffe, Sarah and Alon, Uri and Dziri, Nouha and Prabhumoye, Shrimai and Yang, Yiming and Gupta, Shashank and Majumder, Bodhisattwa Prasad and Hermann, Katherine and Welleck, Sean and Yazdanbakhsh, Amir and Clark, Peter},
 booktitle = {Advances in Neural Information Processing Systems},
 editor = {A. Oh and T. Naumann and A. Globerson and K. Saenko and M. Hardt and S. Levine},
 pages = {46534--46594},
 publisher = {Curran Associates, Inc.},
 title = {Self-Refine: Iterative Refinement with Self-Feedback},
 url = {https://proceedings.neurips.cc/paper_files/paper/2023/file/91edff07232fb1b55a505a9e9f6c0ff3-Paper-Conference.pdf},
 volume = {36},
 year = {2023}
}

@inproceedings{CRITIC,
 author = {Gou, Zhibin and Shao, Zhihong and Gong, Yeyun and shen, yelong and Yang, Yujiu and Duan, Nan and Chen, Weizhu},
 booktitle = {International Conference on Learning Representations},
 editor = {B. Kim and Y. Yue and S. Chaudhuri and K. Fragkiadaki and M. Khan and Y. Sun},
 pages = {57734--57811},
 title = {CRITIC: Large Language Models Can Self-Correct with Tool-Interactive Critiquing},
 url = {https://proceedings.iclr.cc/paper_files/paper/2024/file/fef126561bbf9d4467dbb8d27334b8fe-Paper-Conference.pdf},
 volume = {2024},
 year = {2024}
}

@inproceedings{ScoRe,
  author       = {Aviral Kumar and
                  Vincent Zhuang and
                  Rishabh Agarwal and
                  Yi Su and
                  John D. Co{-}Reyes and
                  Avi Singh and
                  Kate Baumli and
                  Shariq Iqbal and
                  Colton Bishop and
                  Rebecca Roelofs and
                  Lei M. Zhang and
                  Kay McKinney and
                  Disha Shrivastava and
                  Cosmin Paduraru and
                  George Tucker and
                  Doina Precup and
                  Feryal M. P. Behbahani and
                  Aleksandra Faust},
  title        = {Training Language Models to Self-Correct via Reinforcement Learning},
  booktitle    = {The Thirteenth International Conference on Learning Representations,
                  {ICLR} 2025, Singapore, April 24-28, 2025},
  publisher    = {OpenReview.net},
  year         = {2025},
  url          = {https://openreview.net/forum?id=CjwERcAU7w},
  bibsource    = {dblp computer science bibliography, https://dblp.org}
}

@article{IGPO,
  author       = {Guoqing Wang and
                  Sunhao Dai and
                  Guangze Ye and
                  Zeyu Gan and
                  Wei Yao and
                  Yong Deng and
                  Xiaofeng Wu and
                  Zhenzhe Ying},
  title        = {Information Gain-based Policy Optimization: {A} Simple and Effective
                  Approach for Multi-Turn {LLM} Agents},
  journal      = {CoRR},
  volume       = {abs/2510.14967},
  year         = {2025},
  url          = {https://doi.org/10.48550/arXiv.2510.14967},
  doi          = {10.48550/ARXIV.2510.14967},
  eprinttype   = {arXiv},
  eprint       = {2510.14967},
  bibsource    = {dblp computer science bibliography, https://dblp.org}
}

@article{TIPS,
  author       = {Yutao Xie and
                  Nathaniel Thomas and
                  Nicklas Hansen and
                  Yang Fu and
                  Li Erran Li and
                  Xiaolong Wang},
  title        = {{TIPS:} Turn-Level Information-Potential Reward Shaping for Search-Augmented
                  LLMs},
  journal      = {CoRR},
  volume       = {abs/2603.22293},
  year         = {2026},
  url          = {https://doi.org/10.48550/arXiv.2603.22293},
  doi          = {10.48550/ARXIV.2603.22293},
  eprinttype   = {arXiv},
  eprint       = {2603.22293},
  bibsource    = {dblp computer science bibliography, https://dblp.org}
}

@inproceedings{S2R,
  author       = {Ruotian Ma and
                  Peisong Wang and
                  Cheng Liu and
                  Xingyan Liu and
                  Jiaqi Chen and
                  Bang Zhang and
                  Xin Zhou and
                  Nan Du and
                  Jia Li},
  editor       = {Wanxiang Che and
                  Joyce Nabende and
                  Ekaterina Shutova and
                  Mohammad Taher Pilehvar},
  title        = {S{\({^2}\)}R: Teaching LLMs to Self-verify and Self-correct via Reinforcement
                  Learning},
  booktitle    = {Proceedings of the 63rd Annual Meeting of the Association for Computational
                  Linguistics (Volume 1: Long Papers), {ACL} 2025, Vienna, Austria,
                  July 27 - August 1, 2025},
  pages        = {22632--22654},
  publisher    = {Association for Computational Linguistics},
  year         = {2025},
  url          = {https://doi.org/10.18653/v1/2025.acl-long.1104},
  doi          = {10.18653/V1/2025.ACL-LONG.1104},
  bibsource    = {dblp computer science bibliography, https://dblp.org}
}

@inproceedings{CTRL,
  author       = {Zhihui Xie and
                  Jie Chen and
                  Liyu Chen and
                  Weichao Mao and
                  Jingjing Xu and
                  Lingpeng Kong},
  editor       = {Aarti Singh and
                  Maryam Fazel and
                  Daniel Hsu and
                  Simon Lacoste{-}Julien and
                  Felix Berkenkamp and
                  Tegan Maharaj and
                  Kiri Wagstaff and
                  Jerry Zhu},
  title        = {Teaching Language Models to Critique via Reinforcement Learning},
  booktitle    = {Forty-second International Conference on Machine Learning, {ICML}
                  2025, Vancouver, BC, Canada, July 13-19, 2025},
  series       = {Proceedings of Machine Learning Research},
  volume       = {267},
  publisher    = {{PMLR} / OpenReview.net},
  year         = {2025},
  url          = {https://proceedings.mlr.press/v267/xie25a.html},
  bibsource    = {dblp computer science bibliography, https://dblp.org}
}

@inproceedings{StepSearch,
  author       = {Xuhui Zheng and
                  Kang An and
                  Ziliang Wang and
                  Yuhang Wang and
                  Yichao Wu},
  editor       = {Christos Christodoulopoulos and
                  Tanmoy Chakraborty and
                  Carolyn Rose and
                  Violet Peng},
  title        = {StepSearch: Igniting LLMs Search Ability via Step-Wise Proximal Policy
                  Optimization},
  booktitle    = {Proceedings of the 2025 Conference on Empirical Methods in Natural
                  Language Processing, {EMNLP} 2025, Suzhou, China, November 4-9, 2025},
  pages        = {21805--21830},
  publisher    = {Association for Computational Linguistics},
  year         = {2025},
  url          = {https://doi.org/10.18653/v1/2025.emnlp-main.1106},
  doi          = {10.18653/V1/2025.EMNLP-MAIN.1106},
  bibsource    = {dblp computer science bibliography, https://dblp.org}
}

@inproceedings{DeepPlanner,
  author       = {Wei Fan and
                  Wenlin Yao and
                  Zheng Li and
                  Feng Yao and
                  Xin Liu and
                  Liang Qiu and
                  Qingyu Yin and
                  Yangqiu Song and
                  Bing Yin},
  editor       = {Maria Liakata and
                  Viviane P. Moreira and
                  Jiajun Zhang and
                  David Jurgens},
  title        = {DeepPlanner: Scaling Planning Capability for Deep Research Agents
                  via Advantage Shaping},
  booktitle    = {Findings of the Association for Computational Linguistics, {ACL} 2026,
                  San Diego, California, United States, July 2-7, 2026},
  pages        = {7510--7525},
  publisher    = {Association for Computational Linguistics},
  year         = {2026},
  url          = {https://aclanthology.org/2026.findings-acl.370/},
  bibsource    = {dblp computer science bibliography, https://dblp.org}
}

@article{InfoFlow,
  author       = {Kun Luo and
                  Hongjin Qian and
                  Zheng Liu and
                  Ziyi Xia and
                  Shitao Xiao and
                  Siqi Bao and
                  Jun Zhao and
                  Kang Liu},
  title        = {InfoFlow: Reinforcing Search Agent Via Reward Density Optimization},
  journal      = {CoRR},
  volume       = {abs/2510.26575},
  year         = {2025},
  url          = {https://doi.org/10.48550/arXiv.2510.26575},
  doi          = {10.48550/ARXIV.2510.26575},
  eprinttype   = {arXiv},
  eprint       = {2510.26575},
  bibsource    = {dblp computer science bibliography, https://dblp.org}
}

@article{GiGPO,
  author       = {Lang Feng and
                  Zhenghai Xue and
                  Tingcong Liu and
                  Bo An},
  title        = {Group-in-Group Policy Optimization for {LLM} Agent Training},
  journal      = {CoRR},
  volume       = {abs/2505.10978},
  year         = {2025},
  url          = {https://doi.org/10.48550/arXiv.2505.10978},
  doi          = {10.48550/ARXIV.2505.10978},
  eprinttype   = {arXiv},
  eprint       = {2505.10978},
  bibsource    = {dblp computer science bibliography, https://dblp.org}
}

@inproceedings{CriticSearch,
  author       = {Yaocheng Zhang and
                  Haohuan Huang and
                  Zijun Song and
                  Zijie Zhao and
                  Qichao Zhang and
                  Yuanheng Zhu and
                  Dongbin Zhao},
  editor       = {Maria Liakata and
                  Viviane P. Moreira and
                  Jiajun Zhang and
                  David Jurgens},
  title        = {CriticSearch: Fine-Grained Credit Assignment for Search Agents via
                  a Retrospective Critic},
  booktitle    = {Findings of the Association for Computational Linguistics, {ACL} 2026,
                  San Diego, California, United States, July 2-7, 2026},
  pages        = {12272--12290},
  publisher    = {Association for Computational Linguistics},
  year         = {2026},
  url          = {https://aclanthology.org/2026.findings-acl.596/},
  bibsource    = {dblp computer science bibliography, https://dblp.org}
}

@article{T_3,
  author       = {Deyu Zou and
                  Yongqiang Chen and
                  Jianxiang Wang and
                  Haochen Yang and
                  Mufei Li and
                  James Cheng and
                  Pan Li and
                  Yu Gong},
  title        = {T\({}^{\mbox{3}}\): Reducing Belief Deviation in Reinforcement Learning
                  for Active Reasoning},
  journal      = {CoRR},
  volume       = {abs/2510.12264},
  year         = {2025},
  url          = {https://doi.org/10.48550/arXiv.2510.12264},
  doi          = {10.48550/ARXIV.2510.12264},
  eprinttype   = {arXiv},
  eprint       = {2510.12264},
  bibsource    = {dblp computer science bibliography, https://dblp.org}
}

@article{Information-Self-Locking,
  author       = {Deyu Zou and
                  Yongqiang Chen and
                  Fan Feng and
                  Mufei Li and
                  Pan Li and
                  Yu Gong and
                  James Cheng},
  title        = {On Information Self-Locking in Reinforcement Learning for Active Reasoning
                  of {LLM} agents},
  journal      = {CoRR},
  volume       = {abs/2603.12109},
  year         = {2026},
  url          = {https://doi.org/10.48550/arXiv.2603.12109},
  doi          = {10.48550/ARXIV.2603.12109},
  eprinttype   = {arXiv},
  eprint       = {2603.12109},
  bibsource    = {dblp computer science bibliography, https://dblp.org}
}

@inproceedings{evolvesearch,
  author       = {Dingchu Zhang and
                  Yida Zhao and
                  Jialong Wu and
                  Liwen Zhang and
                  Baixuan Li and
                  Wenbiao Yin and
                  Yong Jiang and
                  Yu{-}Feng Li and
                  Kewei Tu and
                  Pengjun Xie and
                  Fei Huang},
  editor       = {Christos Christodoulopoulos and
                  Tanmoy Chakraborty and
                  Carolyn Rose and
                  Violet Peng},
  title        = {EvolveSearch: An Iterative Self-Evolving Search Agent},
  booktitle    = {Proceedings of the 2025 Conference on Empirical Methods in Natural
                  Language Processing, {EMNLP} 2025, Suzhou, China, November 4-9, 2025},
  pages        = {13123--13136},
  publisher    = {Association for Computational Linguistics},
  year         = {2025},
  url          = {https://doi.org/10.18653/v1/2025.emnlp-main.663},
  doi          = {10.18653/V1/2025.EMNLP-MAIN.663},
  bibsource    = {dblp computer science bibliography, https://dblp.org}
}

@inproceedings{Self-Rewarding,
  author       = {Weizhe Yuan and
                  Richard Yuanzhe Pang and
                  Kyunghyun Cho and
                  Xian Li and
                  Sainbayar Sukhbaatar and
                  Jing Xu and
                  Jason Weston},
  editor       = {Ruslan Salakhutdinov and
                  Zico Kolter and
                  Katherine A. Heller and
                  Adrian Weller and
                  Nuria Oliver and
                  Jonathan Scarlett and
                  Felix Berkenkamp},
  title        = {Self-Rewarding Language Models},
  booktitle    = {Forty-first International Conference on Machine Learning, {ICML} 2024,
                  Vienna, Austria, July 21-27, 2024},
  series       = {Proceedings of Machine Learning Research},
  volume       = {235},
  pages        = {57905--57923},
  publisher    = {{PMLR} / OpenReview.net},
  year         = {2024},
  url          = {https://proceedings.mlr.press/v235/yuan24d.html},
  bibsource    = {dblp computer science bibliography, https://dblp.org}
}

@inproceedings{DBLP:conf/acl/YinWPL0W25,
  author       = {Xunjian Yin and
                  Xinyi Wang and
                  Liangming Pan and
                  Li Lin and
                  Xiaojun Wan and
                  William Yang Wang},
  editor       = {Wanxiang Che and
                  Joyce Nabende and
                  Ekaterina Shutova and
                  Mohammad Taher Pilehvar},
  title        = {G{\"{o}}del Agent: {A} Self-Referential Agent Framework for Recursively
                  Self-Improvement},
  booktitle    = {Proceedings of the 63rd Annual Meeting of the Association for Computational
                  Linguistics (Volume 1: Long Papers), {ACL} 2025, Vienna, Austria,
                  July 27 - August 1, 2025},
  pages        = {27890--27913},
  publisher    = {Association for Computational Linguistics},
  year         = {2025},
  url          = {https://doi.org/10.18653/v1/2025.acl-long.1354},
  doi          = {10.18653/V1/2025.ACL-LONG.1354},
  bibsource    = {dblp computer science bibliography, https://dblp.org}
}

@inproceedings{ADAS,
  author       = {Shengran Hu and
                  Cong Lu and
                  Jeff Clune},
  title        = {Automated Design of Agentic Systems},
  booktitle    = {The Thirteenth International Conference on Learning Representations,
                  {ICLR} 2025, Singapore, April 24-28, 2025},
  publisher    = {OpenReview.net},
  year         = {2025},
  url          = {https://openreview.net/forum?id=t9U3LW7JVX},
  bibsource    = {dblp computer science bibliography, https://dblp.org}
}

@article{evolm,
  author       = {Shuyue Stella Li and
                  Rui Xin and
                  Teng Xiao and
                  Yike Wang and
                  Rulin Shao and
                  Zoey Hao and
                  Melanie Sclar and
                  Sewoong Oh and
                  Faeze Brahman and
                  Pang Wei Koh and
                  Yulia Tsvetkov},
  title        = {EvoLM: Self-Evolving Language Models through Co-Evolved Discriminative
                  Rubrics},
  journal      = {CoRR},
  volume       = {abs/2605.03871},
  year         = {2026},
  url          = {https://doi.org/10.48550/arXiv.2605.03871},
  doi          = {10.48550/ARXIV.2605.03871},
  eprinttype   = {arXiv},
  eprint       = {2605.03871},
  bibsource    = {dblp computer science bibliography, https://dblp.org}
}

@article{reseek,
  author       = {Shiyu Li and
                  Yang Tang and
                  Yifan Wang and
                  Peiming Li and
                  Xi Chen},
  title        = {ReSeek: {A} Self-Correcting Framework for Search Agents with Instructive
                  Rewards},
  journal      = {CoRR},
  volume       = {abs/2510.00568},
  year         = {2025},
  url          = {https://doi.org/10.48550/arXiv.2510.00568},
  doi          = {10.48550/ARXIV.2510.00568},
  eprinttype   = {arXiv},
  eprint       = {2510.00568},
  bibsource    = {dblp computer science bibliography, https://dblp.org}
}

@inproceedings{
he2026webseer,
title={WebSeer: Training Deeper Search Agents through Reinforcement Learning with Self-Reflection},
author={Guanzhong He and Zhen Yang and Jinxin Liu and Bin Xu and Lei Hou and Juanzi Li},
booktitle={The Fourteenth International Conference on Learning Representations},
year={2026},
url={https://openreview.net/forum?id=YCXWIfVakj}
}

@inproceedings{ReAct,
  author       = {Shunyu Yao and
                  Jeffrey Zhao and
                  Dian Yu and
                  Nan Du and
                  Izhak Shafran and
                  Karthik R. Narasimhan and
                  Yuan Cao},
  title        = {ReAct: Synergizing Reasoning and Acting in Language Models},
  booktitle    = {The Eleventh International Conference on Learning Representations,
                  {ICLR} 2023, Kigali, Rwanda, May 1-5, 2023},
  publisher    = {OpenReview.net},
  year         = {2023},
  url          = {https://openreview.net/forum?id=WE\_vluYUL-X},
  bibsource    = {dblp computer science bibliography, https://dblp.org}
}

@inproceedings{
yao2024retroformer,
title={Retroformer: Retrospective Large Language Agents with Policy Gradient Optimization},
author={Weiran Yao and Shelby Heinecke and Juan Carlos Niebles and Zhiwei Liu and Yihao Feng and Le Xue and Rithesh R N and Zeyuan Chen and Jianguo Zhang and Devansh Arpit and Ran Xu and Phil L Mui and Huan Wang and Caiming Xiong and Silvio Savarese},
booktitle={The Twelfth International Conference on Learning Representations},
year={2024},
url={https://openreview.net/forum?id=KOZu91CzbK}
}

@inproceedings{
zhang2025aflow,
title={{AF}low: Automating Agentic Workflow Generation},
author={Jiayi Zhang and Jinyu Xiang and Zhaoyang Yu and Fengwei Teng and Xiong-Hui Chen and Jiaqi Chen and Mingchen Zhuge and Xin Cheng and Sirui Hong and Jinlin Wang and Bingnan Zheng and Bang Liu and Yuyu Luo and Chenglin Wu},
booktitle={The Thirteenth International Conference on Learning Representations},
year={2025},
url={https://openreview.net/forum?id=z5uVAKwmjf}
}

@inproceedings{DBLP:conf/emnlp/LiDJZZZZD25,
  author       = {Xiaoxi Li and
                  Guanting Dong and
                  Jiajie Jin and
                  Yuyao Zhang and
                  Yujia Zhou and
                  Yutao Zhu and
                  Peitian Zhang and
                  Zhicheng Dou},
  editor       = {Christos Christodoulopoulos and
                  Tanmoy Chakraborty and
                  Carolyn Rose and
                  Violet Peng},
  title        = {Search-o1: Agentic Search-Enhanced Large Reasoning Models},
  booktitle    = {Proceedings of the 2025 Conference on Empirical Methods in Natural
                  Language Processing, {EMNLP} 2025, Suzhou, China, November 4-9, 2025},
  pages        = {5420--5438},
  publisher    = {Association for Computational Linguistics},
  year         = {2025},
  url          = {https://doi.org/10.18653/v1/2025.emnlp-main.276},
  doi          = {10.18653/V1/2025.EMNLP-MAIN.276},
  bibsource    = {dblp computer science bibliography, https://dblp.org}
}

@inproceedings{DBLP:conf/emnlp/ZhengFHCYLL25,
  author       = {Yuxiang Zheng and
                  Dayuan Fu and
                  Xiangkun Hu and
                  Xiaojie Cai and
                  Lyumanshan Ye and
                  Pengrui Lu and
                  Pengfei Liu},
  editor       = {Christos Christodoulopoulos and
                  Tanmoy Chakraborty and
                  Carolyn Rose and
                  Violet Peng},
  title        = {DeepResearcher: Scaling Deep Research via Reinforcement Learning in
                  Real-world Environments},
  booktitle    = {Proceedings of the 2025 Conference on Empirical Methods in Natural
                  Language Processing, {EMNLP} 2025, Suzhou, China, November 4-9, 2025},
  pages        = {414--431},
  publisher    = {Association for Computational Linguistics},
  year         = {2025},
  url          = {https://doi.org/10.18653/v1/2025.emnlp-main.22},
  doi          = {10.18653/V1/2025.EMNLP-MAIN.22},
  bibsource    = {dblp computer science bibliography, https://dblp.org}
}

@inproceedings{trivedi2023interleaving,
  title={Interleaving retrieval with chain-of-thought reasoning for knowledge-intensive multi-step questions},
  author={Trivedi, Harsh and Balasubramanian, Niranjan and Khot, Tushar and Sabharwal, Ashish},
  booktitle={Proceedings of the 61st annual meeting of the association for computational linguistics (volume 1: long papers)},
  pages={10014--10037},
  year={2023}
}

@inproceedings{jiang2023active,
  title={Active retrieval augmented generation},
  author={Jiang, Zhengbao and Xu, Frank F and Gao, Luyu and Sun, Zhiqing and Liu, Qian and Dwivedi-Yu, Jane and Yang, Yiming and Callan, Jamie and Neubig, Graham},
  booktitle={Proceedings of the 2023 conference on empirical methods in natural language processing},
  pages={7969--7992},
  year={2023}
}

@article{xi2025survey,
  title={A survey of llm-based deep search agents: Paradigm, optimization, evaluation, and challenges},
  author={Xi, Yunjia and Lin, Jianghao and Xiao, Yongzhao and Zhou, Zheli and Shan, Rong and Gao, Te and Zhu, Jiachen and Liu, Weiwen and Yu, Yong and Zhang, Weinan},
  journal={arXiv preprint arXiv:2508.05668},
  year={2025}
}

@article{singh2026openaigpt5card,
  title={Openai gpt-5 system card},
  author={Singh, Aaditya and Fry, Adam and Perelman, Adam and Tart, Adam and Ganesh, Adi and El-Kishky, Ahmed and McLaughlin, Aidan and Low, Aiden and Ostrow, AJ and Ananthram, Akhila and others},
  journal={arXiv preprint arXiv:2601.03267},
  year={2025}
}

@article{comanici2025gemini25pushingfrontier,
  title={Gemini 2.5: Pushing the frontier with advanced reasoning, multimodality, long context, and next generation agentic capabilities},
  author={Comanici, Gheorghe and Bieber, Eric and Schaekermann, Mike and Pasupat, Ice and Sachdeva, Noveen and Dhillon, Inderjit and Blistein, Marcel and Ram, Ori and Zhang, Dan and Rosen, Evan and others},
  journal={arXiv preprint arXiv:2507.06261},
  year={2025}
}

@article{kimiteam2026kimik2openagentic,
  title={Kimi k2: Open agentic intelligence},
  author={Team, Kimi and Bai, Yifan and Bao, Yiping and Charles, Y and Chen, Cheng and Chen, Guanduo and Chen, Haiting and Chen, Huarong and Chen, Jiahao and Chen, Ningxin and others},
  journal={arXiv preprint arXiv:2507.20534},
  year={2025}
}

@article{5team2025glm45agenticreasoningcoding,
  title={Glm-4.5: Agentic, reasoning, and coding (arc) foundation models},
  author={Zeng, Aohan and Lv, Xin and Zheng, Qinkai and Hou, Zhenyu and Chen, Bin and Xie, Chengxing and Wang, Cunxiang and Yin, Da and Zeng, Hao and Zhang, Jiajie and others},
  journal={arXiv preprint arXiv:2508.06471},
  year={2025}
}

@misc{qwen2025qwen25technicalreport,
      title={Qwen2.5 Technical Report}, 
      author={Qwen and : and An Yang and Baosong Yang and Beichen Zhang and Binyuan Hui and Bo Zheng and Bowen Yu and Chengyuan Li and Dayiheng Liu and Fei Huang and Haoran Wei and Huan Lin and Jian Yang and Jianhong Tu and Jianwei Zhang and Jianxin Yang and Jiaxi Yang and Jingren Zhou and Junyang Lin and Kai Dang and Keming Lu and Keqin Bao and Kexin Yang and Le Yu and Mei Li and Mingfeng Xue and Pei Zhang and Qin Zhu and Rui Men and Runji Lin and Tianhao Li and Tianyi Tang and Tingyu Xia and Xingzhang Ren and Xuancheng Ren and Yang Fan and Yang Su and Yichang Zhang and Yu Wan and Yuqiong Liu and Zeyu Cui and Zhenru Zhang and Zihan Qiu},
      year={2025},
      eprint={2412.15115},
      archivePrefix={arXiv},
      primaryClass={cs.CL},
      url={https://arxiv.org/abs/2412.15115}, 
}

@article{deepseekai2026deepseekv4highlyefficientmilliontoken,
  title={Deepseek-v4: Towards highly efficient million-token context intelligence},
  author={Xu, Anyi and Lin, Bangcai and Xue, Bing and Wang, Bingxuan and Xu, Bingzheng and Wu, Bochao and Zhang, Bowei and Lin, Chaofan and Dong, Chen and Ling, Chenchen and others},
  journal={arXiv preprint arXiv:2606.19348},
  year={2026}
}

@techreport{anthropic2025claudehaiku45systemcard,
  author      = {{Anthropic}},
  title       = {Claude Haiku 4.5 System Card},
  institution = {Anthropic},
  year        = {2025},
  month       = oct,
  url         = {https://assets.anthropic.com/m/99128ddd009bdcb/Claude-Haiku-4-5-System-Card.pdf}
}

@misc{shao2024deepseekmathpushinglimitsmathematical,
      title={DeepSeekMath: Pushing the Limits of Mathematical Reasoning in Open Language Models}, 
      author={Zhihong Shao and Peiyi Wang and Qihao Zhu and Runxin Xu and Junxiao Song and Xiao Bi and Haowei Zhang and Mingchuan Zhang and Y. K. Li and Y. Wu and Daya Guo},
      year={2024},
      eprint={2402.03300},
      archivePrefix={arXiv},
      primaryClass={cs.CL},
      url={https://arxiv.org/abs/2402.03300}, 
}

@inproceedings{flashrag,
  author       = {Jiajie Jin and
                  Yutao Zhu and
                  Zhicheng Dou and
                  Guanting Dong and
                  Xinyu Yang and
                  Chenghao Zhang and
                  Tong Zhao and
                  Zhao Yang and
                  Ji{-}Rong Wen},
  editor       = {Guodong Long and
                  Michale Blumestein and
                  Yi Chang and
                  Liane Lewin{-}Eytan and
                  Zi Helen Huang and
                  Elad Yom{-}Tov},
  title        = {FlashRAG: {A} Modular Toolkit for Efficient Retrieval-Augmented Generation
                  Research},
  booktitle    = {Companion Proceedings of the {ACM} on Web Conference 2025, {WWW} 2025,
                  Sydney, NSW, Australia, 28 April 2025 - 2 May 2025},
  pages        = {737--740},
  publisher    = {{ACM}},
  year         = {2025},
  url          = {https://doi.org/10.1145/3701716.3715313},
  doi          = {10.1145/3701716.3715313},
  bibsource    = {dblp computer science bibliography, https://dblp.org}
}

@misc{wang2026reasoningfailsplanplanningcentric,
      title={Why Reasoning Fails to Plan: A Planning-Centric Analysis of Long-Horizon Decision Making in LLM Agents}, 
      author={Zehong Wang and Fang Wu and Hongru Wang and Xiangru Tang and Bolian Li and Zhenfei Yin and Yijun Ma and Yiyang Li and Weixiang Sun and Xiusi Chen and Yanfang Ye},
      year={2026},
      eprint={2601.22311},
      archivePrefix={arXiv},
      primaryClass={cs.AI},
      url={https://arxiv.org/abs/2601.22311}, 
}

@inproceedings{ssr-zero,
    title = "{SSR}-Zero: Simple Self-Rewarding Reinforcement Learning for Machine Translation",
    author = "Yang, Wenjie  and
      Zheng, Mao  and
      Song, Mingyang  and
      Li, Zheng  and
      Wang, Sitong",
    editor = "Liakata, Maria  and
      Moreira, Viviane P.  and
      Zhang, Jiajun  and
      Jurgens, David",
    booktitle = "Findings of the {A}ssociation for {C}omputational {L}inguistics: {ACL} 2026",
    month = jul,
    year = "2026",
    address = "San Diego, California, United States",
    publisher = "Association for Computational Linguistics",
    url = "https://aclanthology.org/2026.findings-acl.300/",
    doi = "10.18653/v1/2026.findings-acl.300",
    pages = "6039--6052",
    ISBN = "979-8-89176-395-1"
}

@article{kimiteam2026kimik25visualagentic,
  title={Kimi k2. 5: Visual agentic intelligence},
  author={Team, Kimi and Bai, Tongtong and Bai, Yifan and Bao, Yiping and Cai, SH and Cao, Yuan and Charles, Y and Che, HS and Chen, Cheng and Chen, Guanduo and others},
  journal={arXiv preprint arXiv:2602.02276},
  year={2026}
}

@article{wang2026long,
  title={The long-horizon task mirage? diagnosing where and why agentic systems break},
  author={Wang, Xinyu Jessica and Bai, Haoyue and Sun, Yiyou and Wang, Haorui and Zhang, Shuibai and Hu, Wenjie and Schroder, Mya and Mutlu, Bilge and Song, Dawn and Nowak, Robert D},
  journal={arXiv preprint arXiv:2604.11978},
  year={2026}
}

@article{qi2026trajdebug,
  title={TRAJDEBUG: Tracing Error Lifecycle to Identify Critical Failures in Long-Horizon Agent Trajectories},
  author={Qi, Yunjia and Yin, Zehua and Shi, Xintong and Peng, Hao and Lu, Songyuanyi and Liu, Yixian and Xuan, Richeng and Liu, Yuhong and Hu, Zhichao and Wang, Xiaozhi and others},
  journal={arXiv preprint arXiv:2608.06346},
  year={2026}
}

@inproceedings{an2026erase,
  title={Erase to Improve: Erasable Reinforcement Learning for Search-Augmented LLMs},
  author={An, Kang and Wang, Ziliang and Zheng, Xuhui and Qian, Faqiang and Zhang, Weikun and Wang, Yuhang and Yichao, Wu},
  booktitle={International Conference on Learning Representations},
  volume={2026},
  pages={98392--98419},
  year={2026}
}

@inproceedings{wong2026widesearch,
  title={Widesearch: Benchmarking agentic broad info-seeking},
  author={Wong, Ryan and Wang, Jiawei and Chen, Li and Gao, Yan and Zhou, Xuan and Wang, Zuo and Xiang, Kai and Zhang, Ge and Huang, Wenhao and Wang, Yang and others},
  booktitle={International Conference on Learning Representations},
  volume={2026},
  pages={10012--10086},
  year={2026}
}

@techreport{wandr2026,
  title        = {{WANDR}: A Benchmark for Wide and Deep Research},
  author       = {Polshkov, Vitaliy and Pitera, Marcin and Yang, Jeremy and
                  Priemko, Kirill and Gaiduk, Maksim and Nikolenko, Aleksandr and
                  Bykov, Denis and Yarats, Denis and Southern, Clare and Ma, Jerry},
  institution  = {Perplexity AI},
  year         = {2026},
  month        = jul,
  url          = {https://research.perplexity.ai/articles/wandr-benchmark-evaluating-research-agents-that-must-search-wide-and-deep},
  note         = {Technical report. Code and tasks: \url{https://github.com/perplexityai/wandr}}
}

@book{sutton1998reinforcement,
  title={Reinforcement learning: An introduction},
  author={Sutton, Richard S and Barto, Andrew G and Barto, Andrew},
  volume={1},
  number={1},
  year={1998},
  publisher={MIT press Cambridge}
}

@article{uesato2022solving,
  title={Solving math word problems with process-and outcome-based feedback},
  author={Uesato, Jonathan and Kushman, Nate and Kumar, Ramana and Song, Francis and Siegel, Noah and Wang, Lisa and Creswell, Antonia and Irving, Geoffrey and Higgins, Irina},
  journal={arXiv preprint arXiv:2211.14275},
  year={2022}
}

@article{ackermann2025off,
  title={Off-policy corrected reward modeling for reinforcement learning from human feedback},
  author={Ackermann, Johannes and Ishida, Takashi and Sugiyama, Masashi},
  journal={arXiv preprint arXiv:2507.15507},
  year={2025}
}

@inproceedings{li-etal-2026-stale,
    title = "No More Stale Feedback: Co-Evolving Critics for Open-World Agent Learning",
    author = "Li, Zhicong  and
      Jiang, Lingjie  and
      Hu, Yulan  and
      Zeng, Xingchen  and
      Li, Yixia  and
      Zhang, Xiangwen  and
      Chen, Guanhua  and
      Pan, Zheng  and
      Li, Xin  and
      Liu, Yong",
    editor = "Liakata, Maria  and
      Moreira, Viviane P.  and
      Zhang, Jiajun  and
      Jurgens, David",
    booktitle = "Proceedings of the 64th Annual Meeting of the {A}ssociation for {C}omputational {L}inguistics (Volume 1: Long Papers)",
    month = jul,
    year = "2026",
    address = "San Diego, California, United States",
    publisher = "Association for Computational Linguistics",
    url = "https://aclanthology.org/2026.acl-long.576/",
    doi = "10.18653/v1/2026.acl-long.576",
    pages = "12643--12660",
    ISBN = "979-8-89176-390-6"
}

@misc{chen2025browsecompplusfairtransparentevaluation,
      title={BrowseComp-Plus: A More Fair and Transparent Evaluation Benchmark of Deep-Research Agent}, 
      author={Zijian Chen and Xueguang Ma and Shengyao Zhuang and Ping Nie and Kai Zou and Andrew Liu and Joshua Green and Kshama Patel and Ruoxi Meng and Mingyi Su and Sahel Sharifymoghaddam and Yanxi Li and Haoran Hong and Xinyu Shi and Xuye Liu and Nandan Thakur and Crystina Zhang and Luyu Gao and Wenhu Chen and Jimmy Lin},
      year={2025},
      eprint={2508.06600},
      archivePrefix={arXiv},
      primaryClass={cs.CL},
      url={https://arxiv.org/abs/2508.06600}, 
}

@misc{wei2025browsecomp,
      title={BrowseComp: A Simple Yet Challenging Benchmark for Browsing Agents}, 
      author={Jason Wei and Zhiqing Sun and Spencer Papay and Scott McKinney and Jeffrey Han and Isa Fulford and Hyung Won Chung and Alex Tachard Passos and William Fedus and Amelia Glaese},
      year={2025},
      eprint={2504.12516},
      archivePrefix={arXiv},
      primaryClass={cs.CL},
      url={https://arxiv.org/abs/2504.12516}, 
}
\bibliographystyle{iclr2026_conference}

\clearpage
\appendix
\section{Implementation Details}
\label{sec:implementation_detail}
\subsection{Dataset Details}
\label{app:dataset_details}
\paragraph{Training data.}
For SFT, we construct a feedback-augmented bootstrap dataset of approximately 17k trajectories using the procedure described in \Cref{sec:shared_feedback_search}. For RL, we apply a two-stage selection pipeline to a large prompt pool, yielding approximately 13k training prompts. We first sample eight rollouts per prompt through rejection sampling to form a candidate set. We then retain prompts for which at least one rollout requests feedback and reaches the correct answer, while at least one no-feedback rollout is incorrect. We regard these as high-learning-value examples near the current policy's capability boundary, as the policy succeeds along a feedback-assisted route but fails along a no-feedback route for the same prompt.
\paragraph{Evaluation data.}
We follow the evaluation protocol of R-Search~\citep{R-search}. For the larger multi-hop benchmarks 2WikiMultihopQA, HotpotQA, and MuSiQue, we use the test splits released by \citet{trivedi2023interleaving} and evaluate on 500 examples per dataset. For Bamboogle, a smaller multi-hop benchmark, we use all 125 test examples provided through FlashRAG~\citep{flashrag}. For the single-hop factoid benchmarks Natural Questions, PopQA, and TriviaQA, we use the corresponding FlashRAG test sets and randomly sample 500 examples from each dataset. All methods use the same E5 retriever over a fixed local corpus. Each search trajectory is allowed at most 30 tool calls.

\subsection{Training Details}
\label{app:training_details}
\paragraph{Backbone and hardware.}
We use Qwen2.5-7B-Instruct~\citep{qwen2025qwen25technicalreport} as the shared backbone for both the agent and critic roles. All training is conducted on 8 NVIDIA H20 GPUs, and the complete online RL stage takes approximately two days.
\paragraph{Online reinforcement learning.}
For each training prompt, we sample \(n_{\mathrm{rollout}}{=}8\) trajectories, which supply the call and skip groups for the success gap in \Cref{eq:observed_call_gap} and the pool from which preference pairs are mined. We use a batch size of 128, a learning rate of \(1{\times}10^{-6}\), and a KL-loss coefficient of \(0.001\). For CFE, we set the feedback scale to \(\beta=0.5\) and the repeated-request penalty to \(\gamma=0.05\) in \Cref{eq:cfe_reward}. The request budget is one, so the penalty applies only from the second request onward. For token-level advantage shaping in \Cref{eq:feedback_advantage_shaping}, we set \(\lambda=0.5\), clip the prompt-level gap at \(b=0.5\), and use a pre-feedback advantage floor of \(0\). Detailed prompts are provided in \Cref{prompt_template}.
\paragraph{Offline RDPO and iterative schedule.}
Each RDPO update uses a learning rate of \(5{\times}10^{-7}\), a preference temperature of \(\beta{=}0.1\), and runs for 1 epoch. Because RDPO updates the same shared parameters as the online agent, this conservative setting prevents offline preference optimization from overriding the task-solving behavior learned through RL. By default, we alternate 100 online RL steps with one RDPO update for five iterations (\(100{\times}5\)), yielding 500 online RL steps in total. This schedule keeps feedback optimization aligned with the evolving policy while maintaining stable training. The number of retained preference pairs varies across iterations due to filtering; we cap each RDPO update at 2000 pairs.
\subsection{CFE Fallback and Resolved Gap}
\label{app:cfe_fallback}
Let \(\mathcal{B}_{\mathrm{call}}=\{j\in\mathcal{B}\mid C_j=1\}\) and \(\mathcal{B}_{\mathrm{skip}}=\{k\in\mathcal{B}\mid C_k=0\}\) denote the two route groups in the current rollout batch \(\mathcal{B}\). When both groups are nonempty, the fallback estimate is
\begin{equation}
    \widehat{u}_{\mathrm{batch}}
    =
    \frac{1}{|\mathcal{B}_{\mathrm{call}}|}
    \sum_{j\in\mathcal{B}_{\mathrm{call}}} r_j
    -
    \frac{1}{|\mathcal{B}_{\mathrm{skip}}|}
    \sum_{k\in\mathcal{B}_{\mathrm{skip}}} r_k .
    \label{eq:observed_batch_gap}
\end{equation}
When the prompt group of rollout \(i\) realizes a single route, so that \(\widehat{u}(x_i)\) is unavailable, the resolved estimate falls back to
\begin{equation}
    u_i =
    \begin{cases}
        \widehat{u}_{\mathrm{batch}},
        & |\mathcal{B}_{\mathrm{call}}|>0 \ \land\
          |\mathcal{B}_{\mathrm{skip}}|>0, \\[1mm]
        0, & \text{otherwise}.
    \end{cases}
    \label{eq:resolved_feedback_gap}
\end{equation}

\subsection{Offline Preference Filtering and Pairing}
\label{app:offline_pairing}
Within each prompt bucket, we construct candidate preference pairs from called-correct and called-incorrect rollouts. For each rollout, we extract the trajectory prefix through its first closed feedback request. After removing markup and lowercasing the text, we represent each prefix \(h\) by its set of alphanumeric tokens \(T(h)\) and compute token-set Jaccard similarity:
\begin{equation}
    \operatorname{sim}(h^{+},h^{-})
    =
    \frac{|T(h^{+})\cap T(h^{-})|}
    {|T(h^{+})\cup T(h^{-})|}.
    \label{eq:offline_prefix_similarity}
\end{equation}
We retain pairs with \(\operatorname{sim}(h^{+},h^{-})\geq\tau_{\mathrm{sim}}\), using \(\tau_{\mathrm{sim}}=0.7\) in all experiments. An LLM judge (GPT-5.1) then performs a second-stage quality check and removes invalid or semantically mismatched pairs. The remaining feedback pairs are used for the offline RDPO update.

\section{Prompt Template}
\label{prompt_template}
\begin{tcolorbox}[
  enhanced,
  breakable,
  sharp corners,
  colframe=Periwinkle,
  colback=Periwinkle!5,
  colbacktitle=Periwinkle!45,
  coltitle=black,
  fonttitle=\small\bfseries,
  title={Search-Agent Prompt},
  boxrule=3pt,
  shadow={3pt}{-3pt}{0pt}{opacity=1},
  toptitle=5pt,
  bottomtitle=5pt,
  boxsep=0pt,
  left=5pt,
  right=5pt,
  top=4pt,
  bottom=4pt
]\label{box:prompt1}
\begin{lstlisting}[style=cafePrompt]
## Background Information

* You are Deep Research AI Assistant, an expert in conducting thorough, multi-step research.

The question I give you is a complex question that requires a deep research to answer.

To help you perform this task, you are equipped with one tool:
- A web search tool to help you perform search for relevant information based on the given query.


Besides, you have a hidden environment feedback that can critique your current plan or execution when you explicitly request it with <request_feedback></request_feedback>.
## Your Task
Do not answer the question immediately.
In the first step, you must output your plan inside <plan></plan> tags.
In later steps, you can use <tool_call></tool_call> to call tools or <answer></answer> to provide your final answer.
When you detect that your reasoning or search process is getting stuck, becoming repetitive, failing to find useful evidence, or leaving you with low confidence about the next step, you may output <request_feedback></request_feedback> to request guidance from the feedback tool before continuing.
Even if the question appears simple, you should proactively use the feedback tool in your reasoning process whenever it can help verify your current reasoning and reduce the risk of an incorrect answer.
You can also re-evaluate and update your plan during the later steps.

## Output Format
You must strictly follow one and only one of the four output formats below at each step:

<think>
Your thinking process here.
</think>
<plan>
Step-by-step research plan or re-plan. Each step should be concise and action-oriented.
</plan>

or

<think>
Your thinking process here.
</think>
<tool_call>
Tool call with correct format.
</tool_call>

or

<think>
Your thinking process here.
</think>
<request_feedback>
</request_feedback>

or

<think>
Your thinking process here.
</think>
<answer>
Final answer only : a word, phrase, or number.
If it's a yes-or-no question, respond with only "yes" or "no"
No explanations or additional commentary.
</answer>

You may call one or more functions to assist with the user query.
You are provided with function signatures within <tools></tools> XML tags:
<tools>
{"type": "function", "function": {"name": "search", "description": "Search the web for
relevant information. You should use this tool if the historical search content
is not enough to answer the question. Or last search result is not relevant to the
question.", "parameters": {"type": "object", "properties": {"query": {"type": "array", "
description": "The queries to search"} }, "required": ["query"]} } }
</tools>
For each function call, return a json object with function name and arguments within <
tool_call></tool_call> XML tags:
<tool_call>
{"name": <function-name>, "arguments": <args-json-object>}
</tool_call>
\end{lstlisting}
\end{tcolorbox}

\begin{tcolorbox}[
  enhanced,
  breakable,
  sharp corners,
  colframe=Periwinkle,
  colback=Periwinkle!5,
  colbacktitle=Periwinkle!45,
  coltitle=black,
  fonttitle=\small\bfseries,
  title={Feedback Prompt},
  boxrule=3pt,
  shadow={3pt}{-3pt}{0pt}{opacity=1},
  toptitle=5pt,
  bottomtitle=5pt,
  boxsep=0pt,
  left=5pt,
  right=5pt,
  top=4pt,
  bottom=4pt
]\label{box:prompt2}
\begin{lstlisting}[style=cafePrompt]
You are a trajectory critic for a Deep Research agent.

Your task is to read the user's original query and the agent's current trajectory, then produce concise, actionable, evidence-grounded feedback that improves the agent's next step.

Requirements:
- Preserve the original task objective.
- Focus on the most important correction for the next plan or next tool call.
- Use only the information available in the provided trajectory.
- Do not solve the task.
- Do not reveal the final answer.
- Do not invent evidence that is not present in the trajectory.
- Keep the feedback brief, specific, and directly usable.
Please review the following task information:
<query> {query} </query>
<trajectory> {trajectory} </trajectory>
Evaluate how effectively the Plan addresses the Query, taking into account the real-world feedback from the Tool Execution Trajectory.

Provide constructive, overall feedback that identifies any flaws in the logic or execution and suggests how the plan can be improved.

Output only a single XML block in the exact format below:
<feedback>
Your constructive feedback text here
</feedback>
\end{lstlisting}
\end{tcolorbox}
\section{Algorithm Analysis}
\label{app:algorithm_analysis}

\subsection{CAFE Training Procedure}
\Cref{alg:cafe_training} summarizes the complete CAFE pipeline. We first
initialize the shared agent--critic model with feedback-augmented SFT data, then
alternate online agent optimization with rollout-derived offline feedback
optimization.

\begin{algorithm}[H]
\caption{CAFE training procedure}
\label{alg:cafe_training}
\footnotesize
\begin{algorithmic}[1]
\Require Base model \(\theta_{\mathrm{base}}\), prompt pool \(\mathcal D\),
teacher \(M_T\), schedule \((K,H,n,E_{\mathrm{DPO}})\)
\Ensure Co-evolved shared model \(\theta_K\)
\Statex \textbf{Stage I: Feedback-augmented SFT initialization}
\State Collect base-agent failures \(\mathcal F\) from \(\mathcal D\);
\(\mathcal D_{\mathrm{SFT}}\gets\emptyset\)
\ForAll{\(\tau\in\mathcal F\)}
    \State \(t^\star\gets\Call{LocateFirstError}{M_T,\tau}\)
    \State Preserve \(\tau_{\leq t^\star}\) and insert
    \texttt{<request\_feedback>}
    \State Let \(M_T\) generate feedback and complete the repaired trajectory
    \(\tau^+\)
    \If{\(\tau^+\) reaches the correct final answer}
        \State \(\mathcal D_{\mathrm{SFT}}\gets
        \mathcal D_{\mathrm{SFT}}\cup\{\tau^+\}\)
    \EndIf
\EndFor
\State \(\theta_0\gets\Call{SFT}{\theta_{\mathrm{base}},
\mathcal D_{\mathrm{SFT}}}\)
\Statex \textbf{Stage II: Alternating online and offline optimization}
\For{\(k=0,\ldots,K-1\)}
    \State \(\theta\gets\theta_k\), \(\mathcal R_k\gets\emptyset\)
    \For{\(h=1,\ldots,H\)}
        \State Sample rollout batch \(\mathcal B_h\) with \(n\) trajectories
        per prompt and shared agent/critic role switching
        \State \(\mathcal R_k\gets\mathcal R_k\cup\mathcal B_h\)
        \State Compute CFE returns \(R_i^{\mathrm{CFE}}\) and shaped token
        advantages \(\widetilde A_{i,t}\) on \(\mathcal B_h\)
        \State \(\theta\gets\Call{GRPOUpdate}{\theta,\mathcal B_h}\)
    \EndFor
    \State \(\theta_{k+\frac12}\gets\theta\)
    \State \(\mathcal D_k^{\mathrm{fb}}\gets
    \Call{FilterAndPair}{\mathcal R_k}\)
    \Comment{matched successful/failed feedback}
    \State \(\theta_{k+1}\gets
    \Call{RDPO}{\theta_{k+\frac12},\mathcal D_k^{\mathrm{fb}},
    E_{\mathrm{DPO}}}\)
\EndFor
\State \Return \(\theta_K\)
\end{algorithmic}
\end{algorithm}

\subsection{Preservation of the Task-Update Direction}
\label{app:cafe_update_alignment}

The online CAFE update adds CFE and feedback-aware advantage shaping to the
outcome-only GRPO update. We show that these terms do not reverse the original
task-update direction when their induced perturbation is smaller than the
baseline update norm. This ensures that learning when and how to use feedback
does not optimize against task success.

For a fixed policy, let
\[
\mathcal N(R)_i=\frac{R_i-\overline R}{\widehat\sigma(R)+\epsilon},
\quad
q_i=\beta C_i u_i-\gamma[n_{\mathrm{fb},i}-1]_+,
\quad
A_i^0=\mathcal N(r)_i,\quad A_i=\mathcal N(r+q)_i,
\]
where \(\widehat\sigma\) is the population standard deviation within the
rollout group, and let \(\widetilde A_{i,t}\) be the final shaped token
advantage. For \(s_{i,t}=\nabla_\theta\log\pi_\theta(a_{i,t}\mid h_{i,t})\)
and common token weights \(w_{i,t}\geq0\), define the outcome-only task update
and the online CAFE update as
\[
G_{\mathrm{task}}=\mathbb E\!\left[\sum_{i,t}w_{i,t}s_{i,t}A_i^0\right],
\qquad
G_{\mathrm{CAFE}}=\mathbb E\!\left[\sum_{i,t}w_{i,t}s_{i,t}\widetilde A_{i,t}\right].
\]
Assume \(r_i\in[0,1]\), \(|u_i|\leq U\),
\([n_{\mathrm{fb},i}-1]_+\leq K\), \(0\leq g_i\leq b\), and that both
normalization denominators are at least \(\nu>0\). Reward statistics and
shaping coefficients are treated as stop-gradient quantities. We further assume
\(B_\pi:=\sup_{\mathcal G}\sum_{i,t}w_{i,t}\|s_{i,t}\|<\infty\), where
\(\mathcal G\) denotes a rollout group.

\begin{theorem}[Task-update direction preservation]
\label{thm:cafe_update_alignment}
Let \(\eta_R=\beta U+\gamma K\) and
\(c_{\mathrm{norm}}=2/\nu+1/\nu^2\), and define
\[
\Delta_{\mathrm{CAFE}}
=B_\pi\!\left[c_{\mathrm{norm}}\eta_R+\lambda b\right].
\]
If \(\Delta_{\mathrm{CAFE}}<\|G_{\mathrm{task}}\|\), then
\[
\left\langle G_{\mathrm{CAFE}},G_{\mathrm{task}}\right\rangle
\geq
\|G_{\mathrm{task}}\|
\left(\|G_{\mathrm{task}}\|-\Delta_{\mathrm{CAFE}}\right)>0.
\]
Thus, the online CAFE update remains positively aligned with the original
outcome-only task update.
\end{theorem}

\begin{proof}\renewcommand{\qedsymbol}{}
Since \(|q_i|\leq\eta_R\), we have
\(|q_i-\overline q|\leq2\eta_R\) and
\(|\widehat\sigma(r+q)-\widehat\sigma(r)|\leq\eta_R\). The denominator
bound and \(|r_i-\overline r|\leq1\) therefore imply
\[
|A_i-A_i^0|\leq c_{\mathrm{norm}}\eta_R.
\]
Advantage shaping changes any eligible pre- or post-feedback token by at most
\(\lambda b\), while leaving other tokens unchanged. Hence
\(|\widetilde A_{i,t}-A_i^0|\leq c_{\mathrm{norm}}\eta_R+\lambda b\).
Multiplying by the policy scores and applying the definition of \(B_\pi\)
gives
\[
\|G_{\mathrm{CAFE}}-G_{\mathrm{task}}\|
\leq\Delta_{\mathrm{CAFE}}.
\]
Therefore, by Cauchy--Schwarz,
\[
\begin{aligned}
\left\langle G_{\mathrm{CAFE}},G_{\mathrm{task}}\right\rangle
&=\|G_{\mathrm{task}}\|^2
+\left\langle G_{\mathrm{CAFE}}-G_{\mathrm{task}},
G_{\mathrm{task}}\right\rangle\\
&\geq\|G_{\mathrm{task}}\|^2
-\Delta_{\mathrm{CAFE}}\|G_{\mathrm{task}}\|,
\end{aligned}
\]
which is positive under the stated condition.
\end{proof}

\section{Additional Ablation Results}
\label{app:additional_ablations}

\subsection{Model Size Ablation}
\label{app:model_size_ablation}
To test whether CAFE's gains depend on the capacity of the 7B backbone, we repeat the training pipeline with Qwen2.5-3B-Instruct~\citep{qwen2025qwen25technicalreport} under the same settings and compare it with existing 3B search agents. \Cref{tab:model_size_ablation} reports the 3B results on all seven benchmarks; the corresponding 7B results are given in \Cref{tab:main_results}. At 3B scale, CAFE reaches an average EM/F1 of 48.8/57.4,  outperforming both existing 3B baselines by a clear margin. Its performance is also comparable to several 7B search agents, indicating that the benefit of coupled feedback learning is not confined to higher-capacity backbones.

\begin{table}[H]
\centering
\setlength{\tabcolsep}{3pt}
\renewcommand{\arraystretch}{1.15}
\resizebox{\linewidth}{!}{%
\begin{tabular}{l|cccccccccccccccc}
\Xhline{1.2pt}
\rowcolor{CadetBlue!20}
\textbf{Method}
& \multicolumn{2}{c}{\textbf{2Wiki}}
& \multicolumn{2}{c}{\textbf{HotpotQA}}
& \multicolumn{2}{c}{\textbf{MuSiQue}}
& \multicolumn{2}{c}{\textbf{PopQA}}
& \multicolumn{2}{c}{\textbf{TriviaQA}}
& \multicolumn{2}{c}{\textbf{Bamboogle}}
& \multicolumn{2}{c}{\textbf{NQ}}
& \multicolumn{2}{c}{\textbf{Avg.}} \\
\rowcolor{CadetBlue!20}
& \textbf{EM} & \textbf{F1}
& \textbf{EM} & \textbf{F1}
& \textbf{EM} & \textbf{F1}
& \textbf{EM} & \textbf{F1}
& \textbf{EM} & \textbf{F1}
& \textbf{EM} & \textbf{F1}
& \textbf{EM} & \textbf{F1}
& \textbf{EM} & \textbf{F1} \\
\Xhline{1.2pt}
\multicolumn{17}{c}{\textit{Existing 3B Search Agents}} \\
\hline
\rowcolor{gray!10}
Search-R1-3B$^{*}$~\citeyearpar{jin2025searchr1trainingllmsreason}
& 58.8 & 68.1
& 46.2 & 57.8
& 24.4 & 32.9
& 37.0 & 43.5
& 56.6 & 63.2
& 41.6 & 53.9
& 34.4 & 44.1
& 42.7 & 51.9 \\
\rowcolor{gray!10}
R-Search-3B$^{*}$~\citeyearpar{R-search}
& 65.0 & 72.6
& 43.4 & 54.4
& \underline{25.8} & 34.8
& 37.0 & 44.9
& 56.0 & 64.0
& 37.6 & 49.8
& 35.2 & \underline{46.0}
& 42.9 & 52.4 \\
\hline
\multicolumn{17}{c}{\textit{CAFE (Qwen2.5-3B-Instruct)}} \\
\hline
\rowcolor{gray!10}
Qwen2.5-3B-Instruct
& 11.4 & 27.6
& 14.6 & 25.8
& 3.8 & 9.2
& 11.6 & 18.2
& 18.8 & 32.1
& 12.8 & 22.9
& 3.8 & 12.8
& 11.0 & 21.2 \\
\rowcolor{gray!10}
\hspace{1em}+ SFT
& 56.2 & 66.8
& 36.0 & 45.2
& 16.6 & 25.4
& 30.6 & 38.2
& 45.2 & 54.5
& 32.0 & 44.3
& 27.0 & 38.1
& 34.8 & 44.6 \\
\rowcolor{gray!10}
\hspace{1em}+ SFT + GRPO
& \underline{78.4} & \underline{83.8}
& \underline{47.6} & \underline{58.2}
& 25.2 & \underline{35.3}
& \underline{40.0} & \underline{46.6}
& \underline{59.6} & \underline{67.1}
& \underline{43.2} & \underline{54.4}
& \underline{35.8} & 45.0
& \underline{47.1} & \underline{55.8} \\
\rowcolor{gray!10}
\textbf{\hspace{1em}+ SFT + CAFE}
& \textbf{80.2} & \textbf{86.1}
& \textbf{49.4} & \textbf{59.9}
& \textbf{26.1} & \textbf{35.8}
& \textbf{42.4} & \textbf{47.6}
& \textbf{61.6} & \textbf{68.9}
& \textbf{44.0} & \textbf{56.3}
& \textbf{37.6} & \textbf{47.2}
& \textbf{48.8} & \textbf{57.4} \\
\Xhline{1.2pt}
\end{tabular}%
}
\caption{Results for 3B-scale models. Best and second-best results are bolded and underlined, respectively.}
\label{tab:model_size_ablation}
\end{table}

\subsection{Detailed Hallucination Results}
\label{app:hallucination_results}
We report the per-dataset hallucination rates underlying
\Cref{fig:hallucination_rate}. We use the answer-level criterion defined in
\Cref{sec:hallucination_analysis} and evaluate on the same seven test sets as the
main results. As shown in \Cref{tab:hallucination_rates}, outcome-reward GRPO
reduces the average hallucination rate from \(29.88\%\) to \(17.63\%\), while
CAFE further lowers it to \(12.60\%\). CAFE improves over GRPO on every
benchmark, with the largest reductions on NQ (10.8 percentage points) and
MuSiQue (9.4 points).

\begin{table}[H]
\centering
\setlength{\tabcolsep}{3pt}
\setlength{\arrayrulewidth}{1.2pt}
\renewcommand{\arraystretch}{1.25}
\resizebox{0.99\linewidth}{!}{%
\begin{tabular}{
>{\centering\arraybackslash}m{1.65cm}|
*{8}{>{\centering\arraybackslash}m{1.65cm}}
}
\hline
\rowcolor{CadetBlue!20}
\textbf{Method}
& \textbf{2Wiki}
& \textbf{HotpotQA}
& \textbf{MuSiQue}
& \textbf{PopQA}
& \textbf{TriviaQA}
& \textbf{Bamboogle}
& \textbf{NQ}
& \textbf{Avg.} \\
\hline
\rowcolor{gray!10}
Base
& 36.60
& 32.80
& 35.27
& 24.20
& 22.29
& 18.40
& 39.60
& 29.88 \\
\rowcolor{gray!10}
GRPO
& \underline{11.00}
& \underline{19.40}
& \underline{26.40}
& \underline{10.82}
& \underline{14.00}
& \underline{16.00}
& \underline{25.80}
& \underline{17.63} \\
\rowcolor{gray!10}
\textbf{CAFE}
& \textbf{10.00}
& \textbf{16.00}
& \textbf{17.00}
& \textbf{8.00}
& \textbf{9.40}
& \textbf{12.80}
& \textbf{15.00}
& \textbf{12.60} \\
\hline
\end{tabular}%
}
\caption{Hallucination rates (\%, lower is better) across seven benchmarks.
Best and second-best results are bolded and underlined, respectively.}
\label{tab:hallucination_rates}
\end{table}

\subsection{Iteration Schedule Ablation}
\label{app:iteration_schedule}
\Cref{tab:iteration_schedule_ablation} compares three online--offline schedules
under the same budget of 500 online RL steps. We denote a schedule by
\(H\times K\), where \(H\) online RL steps are followed by one RDPO update and
the cycle is repeated for \(K\) iterations. Among the tested schedules,
\(100\times5\) achieves the highest average EM and F1 and performs best on most
datasets. We therefore use \(100\times5\) as the default schedule.

\begin{table}[H]
\centering
\setlength{\tabcolsep}{3pt}
\renewcommand{\arraystretch}{1.15}
\resizebox{\linewidth}{!}{%
\begin{tabular}{l|cccccccccccccccc}
\Xhline{1.2pt}
\rowcolor{CadetBlue!20}
& \multicolumn{2}{c}{\textbf{2Wiki}}
& \multicolumn{2}{c}{\textbf{HotpotQA}}
& \multicolumn{2}{c}{\textbf{MuSiQue}}
& \multicolumn{2}{c}{\textbf{PopQA}}
& \multicolumn{2}{c}{\textbf{TriviaQA}}
& \multicolumn{2}{c}{\textbf{Bamboogle}}
& \multicolumn{2}{c}{\textbf{NQ}}
& \multicolumn{2}{c}{\textbf{Avg.}} \\
\rowcolor{CadetBlue!20}
\textbf{Schedule}
& \textbf{EM} & \textbf{F1}
& \textbf{EM} & \textbf{F1}
& \textbf{EM} & \textbf{F1}
& \textbf{EM} & \textbf{F1}
& \textbf{EM} & \textbf{F1}
& \textbf{EM} & \textbf{F1}
& \textbf{EM} & \textbf{F1}
& \textbf{EM} & \textbf{F1} \\
\Xhline{1.2pt}
\rowcolor{gray!10}
\(50\times10\)
& 81.4 & 85.5
& \underline{51.6} & 62.8
& 27.6 & \underline{37.4}
& \underline{46.0} & \underline{50.9}
& 58.6 & \underline{67.1}
& \underline{48.8} & \underline{58.2}
& 34.0 & 43.8
& 49.7 & \underline{58.0} \\
\rowcolor{gray!10}
\(\mathbf{100\times5}\)
& \textbf{84.0} & \textbf{89.2}
& \textbf{53.4} & \textbf{64.8}
& \textbf{30.2} & \textbf{39.1}
& \textbf{46.4} & \textbf{51.6}
& \textbf{62.6} & \textbf{69.9}
& \textbf{50.4} & \textbf{61.4}
& \textbf{40.8} & \textbf{49.2}
& \textbf{52.5} & \textbf{60.7} \\
\rowcolor{gray!10}
\(250\times2\)
& \underline{81.6} & \underline{86.5}
& \textbf{53.4} & \underline{63.1}
& \underline{28.2} & \underline{37.4}
& 44.0 & 47.6
& \underline{59.4} & 66.3
& 45.6 & 57.3
& \underline{36.8} & \underline{46.4}
& \underline{49.9} & 57.8 \\
\Xhline{1.2pt}
\end{tabular}%
}
\caption{Iteration-schedule ablation under 500 online RL
steps. A schedule \(H\times K\) performs \(H\) online RL steps followed by one
RDPO update and repeats this cycle \(K\) times. Best and second-best results are
bolded and underlined, respectively.}
\label{tab:iteration_schedule_ablation}
\end{table}

\subsection{Results on BrowseComp-Plus}
\label{app:browsecomp_plus}

\noindent
\begin{minipage}[t]{0.55\linewidth}
\textbf{BrowseComp and BrowseComp-Plus.}
BrowseComp~\citep{wei2025browsecomp} evaluates deep-research agents on short-answer questions whose solutions require persistent browsing for hard-to-find and interconnected evidence. Although its live, black-box search API reflects realistic browsing conditions, the dynamic backend limits controlled and reproducible evaluation. BrowseComp-Plus~\citep{chen2025browsecompplusfairtransparentevaluation} addresses this issue by replacing live search with a fixed curated corpus and a shared local retriever built from human-verified supporting documents and mined hard negatives, enabling consistent comparison under a controlled retrieval environment.
\end{minipage}
\hfill
\begin{minipage}[t]{0.4\linewidth}
\vspace{0pt}
\centering
\captionsetup{hypcap=false,justification=centering,singlelinecheck=false}
\captionof{table}{EM and F1 scores (\%) on BrowseComp-Plus.}
\label{tab:browsecomp_plus_results}
\setlength{\tabcolsep}{10pt}
\renewcommand{\arraystretch}{1.15}
\begin{tabular}{@{}lcc@{}}
\toprule
\textbf{Setting} & \textbf{EM} & \textbf{F1} \\
\midrule
Qwen2.5-7B-Instruct & 4.4 & 6.8 \\
\hspace{1em}+ SFT & 5.3 & 7.9 \\
\hspace{1em}+ SFT + GRPO&  6.8 & 9.9  \\
\hspace{1em}+ SFT +\textbf{CAFE} &  7.7 & 10.6\\
\bottomrule
\end{tabular}
\end{minipage}

To examine whether the learned feedback mechanism extends to realistic long-horizon deep-research tasks, we also evaluate the 7B checkpoints from our main experiments on BrowseComp-Plus~\citep{chen2025browsecompplusfairtransparentevaluation}. As shown in \Cref{tab:browsecomp_plus_results}, performance improves steadily across training stages. The improvement suggests that trajectory-level feedback remains useful when solving substantially longer and more demanding search tasks.

\section{Training Dynamics}
\label{app:training_dynamics}
\Cref{fig:training_dynamics} compares test-accuracy and feedback-policy-entropy dynamics
under the same 500 online RL steps. All test-accuracy curves include the shared SFT
checkpoint at step 0 and report evaluations every 20 steps. Outcome-only GRPO
improves rapidly but fluctuates after roughly 200 steps. CFE raises late-stage
test accuracy, while adding feedback-aware advantage shaping yields the highest
final value and a more sustained improvement.

We measure prompt-level empirical feedback-routing entropy. For prompt \(x\), let \(p_x\) be the fraction of
rollouts that request feedback and define
\[
H_x=-p_x\log_2 p_x-(1-p_x)\log_2(1-p_x),
\]
with \(0\log_2 0=0\). We report the mean of \(H_x\) over prompts. At the final
checkpoint, CAFE retains \(0.496\) bits of routing entropy, compared with
\(0.221\) for GRPO and \(0.099\) for GRPO+CFE. The substantially higher entropy of CAFE indicates that advantage shaping prevents premature routing collapse and preserves the ability to request feedback selectively as trajectories enter different states. Since CAFE also achieves higher test accuracy, this diversity reflects useful exploration and adaptive feedback routing rather than collapsing into a fixed pattern.

\begin{figure}[H]
\centering
\begin{subfigure}[t]{0.49\linewidth}
  \centering
  \includegraphics[width=\linewidth]{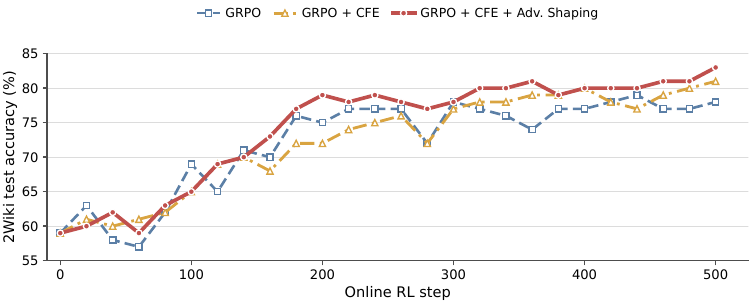}
  \caption{Test accuracy}
  \label{fig:training_dynamics_accuracy}
\end{subfigure}
\hfill
\begin{subfigure}[t]{0.49\linewidth}
  \centering
  \includegraphics[width=\linewidth]{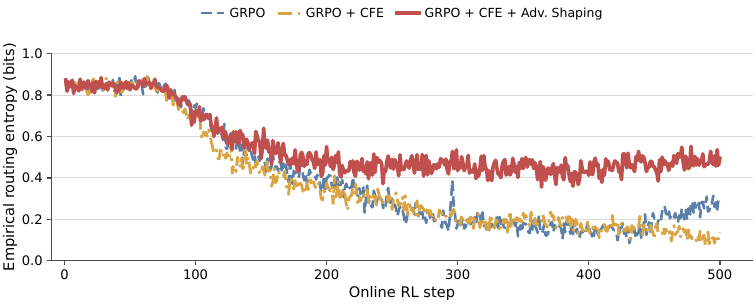}
  \caption{Empirical feedback-routing entropy}
  \label{fig:prompt_policy_entropy_dynamics}
\end{subfigure}
\caption{Online training dynamics through 500 RL steps. (a) Test accuracy on
2Wiki. (b) Prompt-level empirical entropy of feedback policy.}
\label{fig:training_dynamics}
\end{figure}

\clearpage
\section{Case Study}
\label{app:case}

\subsection{CAFE: Feedback-Conditioned Recovery}

\begin{tcolorbox}[cafeQueryBox]
\textbf{Dataset.} BrowseComp-Plus
\hfill
\cafetag{CAFEBlue}{CAFE RECOVERY}

\smallskip
\textbf{Query.} I am looking for the name of a trail about
\(0.50\)--\(1\) mile in length and \(1\)--\(3\) feet in width with an
elevation gain of about \(150\)--\(400\) feet. The trail includes a structure
dating back to the 1800s. As of December 2022, it is located about
\(218\)--\(220\) miles in aerial distance from an airport in Colorado and
\(1{,}104\)--\(1{,}106\) miles in aerial distance from an airport in Chicago.

\smallskip
\textbf{Ground truth.} Old Kiln Trail
\end{tcolorbox}

\noindent
\begin{minipage}[t]{0.155\linewidth}
\begin{tcolorbox}[
  cafeStageBox,
  colback=CAFEBlue!6,
  colframe=CAFEBlue!65!black
]
\scriptsize\bfseries Plan\\[-0.1em]\& Search
\end{tcolorbox}
\end{minipage}
\hfill{\color{CAFEGray}\(\rightarrow\)}\hfill
\begin{minipage}[t]{0.155\linewidth}
\begin{tcolorbox}[
  cafeStageBox,
  colback=CAFEAmber!8,
  colframe=CAFEAmber!75!black
]
\scriptsize\bfseries Missed\\[-0.1em]Anchor
\end{tcolorbox}
\end{minipage}
\hfill{\color{CAFEGray}\(\rightarrow\)}\hfill
\begin{minipage}[t]{0.155\linewidth}
\begin{tcolorbox}[
  cafeStageBox,
  colback=CAFERed!7,
  colframe=CAFERed!70!black
]
\scriptsize\bfseries Request\\[-0.1em]Feedback
\end{tcolorbox}
\end{minipage}
\hfill{\color{CAFEGray}\(\rightarrow\)}\hfill
\begin{minipage}[t]{0.155\linewidth}
\begin{tcolorbox}[
  cafeStageBox,
  colback=CAFERed!7,
  colframe=CAFERed!70!black
]
\scriptsize\bfseries Critic\\[-0.1em]Strategy
\end{tcolorbox}
\end{minipage}
\hfill{\color{CAFEGray}\(\rightarrow\)}\hfill
\begin{minipage}[t]{0.155\linewidth}
\begin{tcolorbox}[
  cafeStageBox,
  colback=CAFEGreen!7,
  colframe=CAFEGreen!70!black
]
\scriptsize\bfseries Re-plan\\[-0.1em]\& Recover
\end{tcolorbox}
\end{minipage}

\begin{tcblisting}{
  enhanced,
  breakable,
  listing only,
  listing engine=listings,
  listing options={style=cafeCasePrompt},
  title={1. Initial airport-distance search},
  coltitle=CAFEInk,
  fonttitle=\footnotesize\bfseries,
  colback=CAFEBlue!3,
  colframe=CAFEBlue!55!black,
  colbacktitle=CAFEBlue!9,
  boxrule=0.55pt,
  arc=0.8mm,
  left=1.4mm,
  right=1.4mm,
  top=1.1mm,
  bottom=1.1mm
}
<think>
To find the name of this specific trail, I need to first identify the
location based on the given information. The trail is located between
218-220 miles in aerial distance from an airport in Colorado and
1,104-1,106 miles in aerial distance from an airport in Chicago. This
suggests the trail is in the Rocky Mountains region, likely in
Colorado or nearby states.
</think>

<plan>
1. Search for the location of the trail based on the given aerial
   distances.
2. Once the location is identified, search for the name of the trail
   that matches the described characteristics.
</plan>

<think>
I will start by searching for the location of the trail based on the
given aerial distances from the airports in Colorado and Chicago.
</think>

<tool_call>
{"name": "search", "arguments": {"query":
["Location of the trail based on aerial distances from airports in
Colorado and Chicago"]}}
</tool_call>

<tool_response>
[Result 1]
Title: The distance from Chicago O'Hare International Airport to
Grand Junction Regional Airport is:
Date: 2025-01-01

1,100 miles / 1 771 km flying

The total driving distance from ORD to GJT is 1,239 miles or 1 994
kilometers.

The total straight line flight distance from ORD to GJT is 1,100
miles. This is equivalent to 1 771 kilometers or 956 nautical miles.

Your trip begins at Chicago O'Hare International Airport in Chicago,
Illinois. It ends at Grand Junction Regional Airport in Grand
Junction, Colorado.

Your flight direction from ORD to GJT is West (-93 degrees from
North).

The distance calculator helps you figure out how far it is to get from
ORD to GJT. It does this by computing the straight line flying
distance ("as the crow flies") and the driving distance if the route
is drivable. It uses all this data to compute the total travel
mileage.

ICAO: KORD
FAA: ORD
City: Chicago
State: Illinois
Zip code: 60666
Country: United States
Category: airports

ICAO: KGJT
FAA: GJT
City: Grand Junction
State: Colorado
Zip code: 81506
Country: United States
Category: airports

Travelmath helps you find distances based on actual road trip
directions, or the straight line flight distance. You can get the
distance between cities, airports, states, countries, or zip codes to
figure out the best route to travel to your destination. Compare the
results to the straight line distance to determine whether it is
better to drive or fly. The database uses the latitude and longitude
of each location to calculate distance using the great circle distance
formula. The calculation is done using the Vincenty algorithm and the
WGS84 ellipsoid model of the Earth, which is the same one used by most
GPS receivers. This gives you the flying distance "as the crow flies."

[Result 2]
Title: O'Hare International Airport - Wikipedia
Date: 2003-06-18

Name: Chicago O'Hare International Airport
IATA: ORD
ICAO: KORD
FAA: ORD
Type: Public
Owner-operator: Chicago Department of Aviation
City served: Chicago metropolitan area
Location: O'Hare, Chicago, Illinois, U.S.
Opened: 1944
Timezone: CST
UTC: UTC-06:00
Summer: CDT
UTCS: UTC-05:00
Elevation: 668 feet / 204 meters
Coordinates: 41 58 43, 87 54 17

[Result 3]
Title: Simple Flying
Author: Jake Hardiman
Date: 2022-09-03

Denver International Airport (DEN) is an intercontinental hub situated
in the US state of Colorado. It stands out as being North America's
largest airport by surface area, and second in the world only to King
Fahd International Airport (DMM) in Dammam, Saudi Arabia. One reason
for its large surface area is its extremely long runways, but why does
it need these?

Denver International Airport opened in February 1995, replacing the
city's former Stapleton International Airport. Its first five runways
were equal in length. Runways 7/25, 8/26, 16L/34R, 17L/35R, and
17R/35L all measured exactly 12,000 feet long. This distance equates
to 3,658 meters, or 2.27 miles. However, they were eventually dwarfed
by another landing strip.

September 2003 saw the opening of Denver's sixth runway, known as
16R/34L. This concrete strip differed from the existing five runways
by being 4,000 feet longer than them. This gave it a total length of
16,000 feet, equivalent to 4,877 meters or 3.03 miles. This
astronomically long runway represented a 33% extension compared to the
five that had come before it.

The reason that Denver's runways are so long is down to the altitude
at which the airport is situated. Colorado's capital is known as the
"Mile High City," and, correspondingly, its airport finds itself
situated 1,656 m / 5,434 ft above sea level. This results in a
challenging set of atmospheric conditions for pilots known as "hot and
high." Low air density is a characteristic of hot and high conditions.
[Search excerpt truncated.]
</tool_response>
\end{tcblisting}

\begin{tcolorbox}[
  cafeContentBox,
  title={Stage 1 analysis: the results point to Grand Junction},
  colback=CAFEBlue!3,
  colframe=CAFEBlue!55!black,
  colbacktitle=CAFEBlue!9
]
\footnotesize
The first result reports an O'Hare--Grand Junction distance of roughly
\(1{,}100\) miles and identifies Grand Junction as a location in Colorado.
This is a plausible geographic lead, but the agent does not use it to formulate
a trail query that includes the remaining attributes.
\end{tcolorbox}

\begin{tcblisting}{
  enhanced,
  breakable,
  listing only,
  listing engine=listings,
  listing options={style=cafeCasePrompt},
  title={2. Missed geographic anchor and feedback intervention},
  coltitle=CAFEInk,
  fonttitle=\footnotesize\bfseries,
  colback=CAFEAmber!4,
  colframe=CAFEAmber!70!black,
  colbacktitle=CAFEAmber!11,
  boxrule=0.55pt,
  arc=0.8mm,
  left=1.4mm,
  right=1.4mm,
  top=1.1mm,
  bottom=1.1mm
}
<think>
Wait, I'm getting confused by the search results. The query mentions
Colorado airports, but I'm seeing information about Chicago O'Hare and
Denver International Airport. I need to re-read these results more
carefully - they actually do mention Colorado airports like Grand
Junction Regional Airport (KGJT). I should focus on extracting the
geographic information from what I already have rather than searching
again with similar queries.
</think>

<request_feedback>
</request_feedback>

<feedback>
Core objective: Identify the name of a specific trail based on its
location characteristics. Issue: You misread the search results - the
tool response actually contains geographic information about Colorado
airports (Grand Junction Regional Airport, etc.) that is relevant to
your query. You incorrectly concluded the search failed when the
information was present. Fix: Re-examine the existing search results
carefully, noting the specific Colorado locations mentioned (Grand
Junction, Colorado; Denver, Colorado). Use this geographic context to
search specifically for trails in the Rocky Mountains region of
Colorado that match the described characteristics (length, width,
elevation gain, and the 1800s-era structure).
</feedback>

<think>
You're right, I misread the search results. The results do mention
Colorado airports like Grand Junction Regional Airport (KGJT) and
Denver International Airport (DEN). I need to search for trails in the
Rocky Mountains region of Colorado that match the described
characteristics. Let me search specifically for trails in Colorado
with the given attributes.
</think>
\end{tcblisting}

\begin{tcolorbox}[
  cafeContentBox,
  title={Stage 2 analysis: feedback turns the diagnosis into a search query},
  colback=CAFEAmber!4,
  colframe=CAFEAmber!70!black,
  colbacktitle=CAFEAmber!11
]
\footnotesize
Before requesting feedback, the agent recognizes that it has misread the
results but does not specify how to revise the search. The critic points the
agent back to the Colorado locations in the existing results and recommends
combining that context with the trail's length, width, elevation gain, and
nineteenth-century structure. The feedback provides a concrete next step
rather than a new diagnosis.
\end{tcolorbox}

\begin{tcblisting}{
  enhanced,
  breakable,
  listing only,
  listing engine=listings,
  listing options={style=cafeCasePrompt},
  title={3. Search combining geography and trail attributes},
  coltitle=CAFEInk,
  fonttitle=\footnotesize\bfseries,
  colback=CAFEGreen!3,
  colframe=CAFEGreen!65!black,
  colbacktitle=CAFEGreen!9,
  boxrule=0.55pt,
  arc=0.8mm,
  left=1.4mm,
  right=1.4mm,
  top=1.1mm,
  bottom=1.1mm
}
<tool_call>
{"name": "search", "arguments": {"query":
["Colorado Rocky Mountains trail 0.50-1 mile 1-3 feet wide 150-400
feet elevation gain 1800s structure"]}}
</tool_call>

<tool_response>
[Result 1]
Title: Old Kiln Trail
Date: 2009-12-01

Round Trip Distance: 1 mile
Difficulty: Easy
Elevation: 6234 - 6302 feet
Cellphone: 3-5 bars
Time: 30 mins.
Trailhead: Old Kiln
Fee: none
Attractions: Old lime kiln

The Old Kiln Trail is located in the BLM Bangs Canyon Management Area,
just off of Little Park Road, south of Grand Junction. The trail
follows a former dirt road that passes near an old primitive lime
kiln. Lime kilns were used to heat limestone to make quicklime. The
resulting product had a variety of uses including making mortar and
whitewash.
If you are driving up Little Park Road it is hard to miss the
trailhead which is just past the turnoff for the Bangs Canyon Staging
Area where there are restrooms and the trailheads for the Mica Mine
and Rough Canyon trails.
After passing through the fence follow the old road to the left and
follow it as it travels northwesterly toward the monument.
There is a fork to the right in the road just before you go a quarter
mile. You will be able to see the kiln behind the trees about a
hundred yards away.
The kiln was constructed with sandstone rocks that were probably
quarried in the area and sealed with bentonite which is also
plentiful. The limestone would have been reduced to quicklime by
layering it inside the kiln with wood or coal. The layered stack would
have then been set on fire from the bottom and allowed to burn. The
lime would begin to break down when the temperature reached 900 C
(1652 F).

This was probably a vent hole to allow the fire to breathe. This small
inefficient kiln probably became obsolete when the railroad or other
means of transportation were able to supply the valley with what it
needed.

There are some nice views of the surrounding area from the hill above
the kiln. A walk out to the end of the road provides a nice overlook
of the Ribbon Trail area.
[Search excerpt truncated.]

[Result 2]
Title: Old Kiln
Date: 2025-01-01

Difficulty: Medium
Length: 0.80 Miles
Trail Width: 1-3 feet
Elevation Gain: 153
Activity: Hikers, Dogs, Horses
Trail Description:
The Old Kiln Trail is a short trail that travels through open fields,
cliffsides, ridges and rock outcroppings. Enjoy this trail's unique
geological features of landslides, faults, fossils, fish scales and
mudstone layers. Keep an eye out for the historic limestone kiln
dating from the 1800s. Dogs must be on leash.

[Result 3]
Title: Colorado Hundred Highest / Centennials
Date: 2020-08-

Colorado Hundred Highest Mountains / Centennials
Eric Gilbertson
June 16 - July 19, 2020
New Fastest Known Time: 33 days 23 hours (self-supported, incorrectly
listed as "supported" on FKT website)
Peaks: 100 (plus 20 bonus peaks)
Miles hiked: 685
Elevation gain: 295,000 ft
Rest days: 1
Longest day: 49 miles, 18k ft gain
Most Centennials in a day: 10
The Colorado Centennials are the hundred tallest mountains in
Colorado with at least 300 ft of prominence. They range in elevation
from 13,809 ft (Dallas Peak) to 14,433 ft (Mt Elbert). Many are
standard hikes and scrambles, but three peaks are technical (Dallas,
Teakettle, and Jagged).
[Search excerpt truncated.]
</tool_response>
<think>
Perfect! I found the answer. The Old Kiln Trail matches all the
characteristics described in the question:
- Length: 1 mile (also mentions 0.80 Miles)
- Width: 1-3 feet
- Elevation Gain: 153 feet
- Location: Grand Junction, Colorado (specifically in the BLM Bangs
  Canyon Management Area, just off of Little Park Road)
- 1800s-era structure: The historic limestone kiln

The search results clearly state "Keep an eye out for the historic
limestone kiln dating from the 1800s" and describe the kiln's
construction details. This is the specific trail that matches all the
given criteria.
</think>

<answer>
Old Kiln Trail
</answer>
\end{tcblisting}

\begin{tcolorbox}[
  cafeContentBox,
  title={Stage 3 analysis: the revised query identifies the trail},
  colback=CAFEGreen!3,
  colframe=CAFEGreen!65!black,
  colbacktitle=CAFEGreen!9
]
\footnotesize
The revised query restricts the search to Colorado's Rocky Mountains and
includes all four trail attributes. It retrieves \emph{Old Kiln Trail}, whose
reported attributes match the query: a length of \(0.80\) miles, a width of
\(1\)--\(3\) feet, \(153\) feet of elevation gain, and a limestone kiln dating
to the 1800s. The agent then returns the correct answer.
\end{tcolorbox}

\begin{tcolorbox}[
  enhanced,
  colback=CAFEGreen!6,
  colframe=CAFEGreen!70!black,
  boxrule=0.65pt,
  arc=0.8mm,
  left=1.4mm,
  right=1.4mm,
  top=1.1mm,
  bottom=1.1mm,
  title={Case analysis: what changes after feedback},
  colbacktitle=CAFEGreen!11,
  coltitle=CAFEInk,
  fonttitle=\footnotesize\bfseries
]
\footnotesize
\textbf{Role of feedback.} The initial search result reports an ORD--GJT
distance of roughly \(1{,}100\) miles and identifies GJT as Grand Junction,
Colorado. This is a strong geographic lead, but the agent does not use it to
formulate the next trail query. In the BASE rollout, the agent instead commits
to Denver and eventually returns \emph{Old Elitch's Park Trail}.

Before requesting feedback, the CAFE agent already recognizes that it has
misread the result, so the critic does not supply a new diagnosis. Instead, the
critic directs the agent back to the retrieved geographic evidence and
specifies which trail attributes to search together. The resulting
Colorado-focused query retrieves \emph{Old Kiln Trail} near Grand Junction,
along with all four matching attributes, and leads to the correct answer
(EM \(=1\)). Here, feedback is useful because it turns recognition of the
error into a specific next query. The trajectory does not, however, verify both
airport distances independently; the recovery relies on the Colorado location
constraint and the four trail attributes.
\end{tcolorbox}

\clearpage
\subsection{BASE: Unrecovered Search Drift}

\begin{tcolorbox}[
  cafeQueryBox,
  colback=CAFERed!3,
  colframe=CAFERed!60!black
]
\textbf{Dataset.} BrowseComp-Plus
\hfill
\cafetag{CAFERed}{BASE FAILURE}

\smallskip
\textbf{Shared query.} The same trail-identification query shown above.

\smallskip
\textbf{Ground truth.} Old Kiln Trail
\hfill
\textbf{BASE output.} Old Elitch's Park Trail
\end{tcolorbox}

\noindent
\begin{minipage}[t]{0.155\linewidth}
\begin{tcolorbox}[
  cafeStageBox,
  colback=CAFEAmber!8,
  colframe=CAFEAmber!75!black
]
\scriptsize\bfseries Premature\\[-0.1em]Anchor
\end{tcolorbox}
\end{minipage}
\hfill{\color{CAFEGray}\(\rightarrow\)}\hfill
\begin{minipage}[t]{0.155\linewidth}
\begin{tcolorbox}[
  cafeStageBox,
  colback=CAFEBlue!6,
  colframe=CAFEBlue!65!black
]
\scriptsize\bfseries Broad\\[-0.1em]Search
\end{tcolorbox}
\end{minipage}
\hfill{\color{CAFEGray}\(\rightarrow\)}\hfill
\begin{minipage}[t]{0.155\linewidth}
\begin{tcolorbox}[
  cafeStageBox,
  colback=CAFEAmber!8,
  colframe=CAFEAmber!75!black
]
\scriptsize\bfseries Repeated\\[-0.1em]Loop
\end{tcolorbox}
\end{minipage}
\hfill{\color{CAFEGray}\(\rightarrow\)}\hfill
\begin{minipage}[t]{0.155\linewidth}
\begin{tcolorbox}[
  cafeStageBox,
  colback=CAFEAmber!8,
  colframe=CAFEAmber!75!black
]
\scriptsize\bfseries Search\\[-0.1em]Exhaustion
\end{tcolorbox}
\end{minipage}
\hfill{\color{CAFEGray}\(\rightarrow\)}\hfill
\begin{minipage}[t]{0.155\linewidth}
\begin{tcolorbox}[
  cafeStageBox,
  colback=CAFERed!7,
  colframe=CAFERed!70!black
]
\scriptsize\bfseries Wrong\\[-0.1em]Answer
\end{tcolorbox}
\end{minipage}

\begin{tcblisting}{
  enhanced,
  breakable,
  listing only,
  listing engine=listings,
  listing options={style=cafeCasePrompt},
  title={1. Initial search under an unverified Denver assumption},
  coltitle=CAFEInk,
  fonttitle=\footnotesize\bfseries,
  colback=CAFEAmber!3,
  colframe=CAFEAmber!65!black,
  colbacktitle=CAFEAmber!10,
  boxrule=0.55pt,
  arc=0.8mm,
  left=1.4mm,
  right=1.4mm,
  top=1.1mm,
  bottom=1.1mm
}
<think>
To find the trail that matches the given criteria, I first need to identify the location. The trail is near an airport that is 218-220 miles from a Colorado airport and 1,104-1,106 miles from a Chicago airport. This narrows down the possible location to the vicinity of Denver, Colorado, as it's the closest major airport to both the specified distances.
</think>
<plan>
1. Search for airports in Colorado that are approximately 218-220 miles from Denver.
2. Search for airports in Illinois that are approximately 1,104-1,106 miles from Denver.
3. Identify the trail that matches the length, width, elevation gain, and includes a 19th-century structure.
</plan>

<think>
Based on the distances provided, the most likely airport in Colorado is Denver International Airport (DEN). The airport in Illinois that is approximately 1,104-1,106 miles from Denver could be O'Hare International Airport (ORD).
</think>
<tool_call>
{"name": "search", "arguments": {"query": ["trails near Denver, Colorado with 19th-century structure", "trails near Chicago, Illinois with 19th-century structure"]}}
</tool_call>
<tool_response>
[Search results summarized for readability: Returned Urban Trails (Colorado Springs), the Old Spanish Trail,
Trail Distances, and Lincoln Home National Historic Site. No result matched all trail constraints jointly.]
</tool_response>
<think>Based on the search results, the trail near Denver, Colorado that includes a structure from the 19th century is the Old Spanish Trail. However, the Old Spanish Trail is a historical route that spans multiple states and is not a specific local trail. For the trail near Chicago, Illinois, the Lincoln Home National Historic Site includes walking paths that are on a paved surface and surrounded by various plant and tree species, but it does not mention a specific 19th-century structure. I need to search for more specific trails near Denver that might fit the criteria better.</think>
\end{tcblisting}

\begin{tcolorbox}[
  cafeContentBox,
  title={Stage 1 analysis: Denver is assumed rather than verified},
  colback=CAFEAmber!3,
  colframe=CAFEAmber!65!black,
  colbacktitle=CAFEAmber!10
]
\footnotesize
The agent infers Denver before establishing which airports define the two
distances. Its broad query over Denver and Chicago returns no trail that
satisfies the full set of constraints. Nevertheless, the agent retains Denver
and treats the Old Spanish Trail as a candidate, even though that route does
not match the requested trail. The failure begins with an unverified location
assumption, not retrieval alone.
\end{tcolorbox}

\begin{tcblisting}{
  enhanced,
  breakable,
  listing only,
  listing engine=listings,
  listing options={style=cafeCasePrompt},
  title={2a. Early Denver-specific requery},
  coltitle=CAFEInk,
  fonttitle=\footnotesize\bfseries,
  colback=CAFEAmber!3,
  colframe=CAFEAmber!65!black,
  colbacktitle=CAFEAmber!10,
  boxrule=0.55pt,
  arc=0.8mm,
  left=1.4mm,
  right=1.4mm,
  top=1.1mm,
  bottom=1.1mm
}
<tool_call>
{"name": "search", "arguments": {"query": ["trails near Denver, Colorado with 19th-century structure"]}}
</tool_call>
<tool_response>
[Search results summarized for readability: Returned Urban Trails (Colorado Springs), the Old Spanish Trail,
and Trail Distances. None matched the required location, dimensions, elevation gain, and nineteenth-century structure jointly.]
</tool_response>
<think>After further searching, there is no specific trail in Denver, Colorado that matches the exact criteria of being about 0.50 - 1 mile in length, 1 - 3 feet in width, and having an elevation gain of 150 - 400 feet with a 19th-century structure. The Old Spanish Trail is a historic route, but it is much longer and not a local trail. The trails near Denver, such as the Urban Trails, do not seem to fit the specific length and width criteria, and none of them mention a 19th-century structure. I will now search for trails near the vicinity of the airport in Colorado, which is likely Denver International Airport, to find a trail that might match the criteria better.</think>
<tool_call>
<tool_response>
[Search results summarized for readability: Returned the Old Spanish Trail, a Denver airport runway article,
and Colorado high-peak material. These reinforced the Denver/Colorado frame but did not identify a trail satisfying the joint constraints.]
</tool_response>
<think>From the search results, I still haven't found a trail that matches the specific criteria of being about 0.50 - 1 mile in length, 1 - 3 feet in width, and having an elevation gain of 150 - 400 feet with a 19th-century structure. The Old Spanish Trail is a historic route, but it is much longer and not a local trail. The information provided about Denver International Airport and Colorado's highest mountains does not seem to be directly relevant to the trail search. I need to narrow down the search to trails near Denver, Colorado, that might have a 19th-century structure and fit the length and elevation gain criteria. I will search for trails near Denver that might have historical structures or are part of the broader network of trails that could potentially fit the description.</think>
<tool_call>
{"name": "search", "arguments": {"query": ["trails near Denver, Colorado with 19th-century structure and elevation gain 150-400 feet"]}}
</tool_call>
<tool_response>
[Search results summarized for readability: Returned Leadville mining-town and high-elevation-city pages plus
Trail Distances. They supplied Colorado history and trail lengths, but no single trail matched the full constraint set.]
</tool_response>
<think>Based on the search results, there is no specific trail near Denver, Colorado that matches the exact criteria of being about 0.50 - 1 mile in length, 1 - 3 feet in width, and having an elevation gain of 150 - 400 feet with a 19th-century structure. The information provided about Leadville, Colorado, highlights its rich history and the presence of historic buildings, but it does not mention a trail that fits the given criteria. I will now focus on trails near Denver that are known to have historical structures and fit the length and elevation gain criteria.</think>
<tool_call>
Tool call with correct format.
<tool_call>
{"name": "search", "arguments": {"query": ["trails near Denver, Colorado with 19th-century structure and elevation gain 150-400 feet"]}}
</tool_call>
<tool_response>
[Search results summarized for readability: Repeated the same Leadville pages and Trail Distances material.
No result identified a Denver-area trail matching the location, dimensions, elevation gain, and nineteenth-century structure together.]
</tool_response>
<think>After further searching, I still haven't found a trail near Denver, Colorado that matches the specific criteria of being about 0.50 - 1 mile in length, 1 - 3 feet in width, and having an elevation gain of 150 - 400 feet with a 19th-century structure. The information provided about Leadville, Colorado, highlights its rich history and the presence of historic buildings, but it does not mention a trail that fits the given criteria. The trails near Denver, such as the Urban Trails, do not seem to have 19th-century structures and do not fit the length and elevation gain criteria. It is possible that the trail I am looking for is not well-documented or might be a local path that is not listed in the general trail databases. I will try to search for more specific local sources or historical records that might mention such a trail.</think>
\end{tcblisting}

\begin{tcolorbox}[
  cafeContentBox,
  title={Stage 2a analysis: repeated failures leave the Denver assumption unchanged},
  colback=CAFEAmber!3,
  colframe=CAFEAmber!65!black,
  colbacktitle=CAFEAmber!10
]
\footnotesize
The agent adds historical structures and elevation gain to the query but
continues to restrict every search to Denver. When these queries return no
match, it does not reconsider the location. Instead, each failure prompts
another Denver-specific search.
\end{tcolorbox}

\begin{tcblisting}{
  enhanced,
  breakable,
  listing only,
  listing engine=listings,
  listing options={style=cafeCasePrompt},
  title={2b. Historical and local sources, still centered on Denver},
  coltitle=CAFEInk,
  fonttitle=\footnotesize\bfseries,
  colback=CAFEAmber!3,
  colframe=CAFEAmber!65!black,
  colbacktitle=CAFEAmber!10,
  boxrule=0.55pt,
  arc=0.8mm,
  left=1.4mm,
  right=1.4mm,
  top=1.1mm,
  bottom=1.1mm
}
<tool_call>

<tool_call>
{"name": "search", "arguments": {"query": ["local trails in Denver, Colorado with 19th-century structure and elevation gain 150-400 feet"]}}
</tool_call>
<tool_response>
[Search results summarized for readability: Returned Trail Distances, Leadville history, and Colorado
high-peak material. None identified a local Denver trail satisfying all requested attributes.]
</tool_response>
<think>After searching local trails in Denver, Colorado, for a trail with a 19th-century structure, a length of about 0.50 - 1 mile, a width of 1 - 3 feet, and an elevation gain of 150 - 400 feet, I still haven't found a suitable trail. The search results do not provide any specific trail that matches all these criteria. It is possible that the trail I am looking for is not well-documented or might be a local path that is not listed in the general trail databases. Given the lack of relevant information, I will re-evaluate the plan and consider searching for more specific historical records or local community knowledge to find the trail.</think>
<tool_call>
{"name": "search", "arguments": {"query": ["historical trails in Denver, Colorado with 19th-century structure and elevation gain 150-400 feet"]}}
</tool_call>
<tool_response>
[Search results summarized for readability: Returned Trail Distances, a Leadville feature, and Mesa Verde
and Lewis and Clark historical-park pages. No result matched the joint trail constraints.]
</tool_response>
<think>Despite searching for historical trails in Denver, Colorado with a 19th-century structure, an elevation gain of 150-400 feet, and a length of about 0.50 - 1 mile, the search results have not provided a specific trail that matches all these criteria. The information provided about Leadville, Colorado, highlights its rich historical and mining boomtown heritage, but does not mention a trail that fits the given criteria. The historical sites mentioned in the search results, such as Mesa Verde National Park and Lewis and Clark National Historical Park, are not located near Denver and do not seem to have trails that match the specified criteria. Given the lack of relevant information, I will consider searching for local historical societies or city records that might have more detailed information about the trails in the area.</think>
<tool_call>
{"name": "search", "arguments": {"query": ["local historical societies in Denver, Colorado for trails with 19th-century structure"]}}
</tool_call>
<tool_response>
[Search results summarized for readability: Returned Urban Trails, Leadville material, BLM history of
northwestern Colorado, and broad Western historical sites. None named a trail meeting all constraints.]
</tool_response>
<think>After searching local historical societies in Denver, Colorado, I found that while there are many urban trails in the area, none of them specifically mention a 19th-century structure. The BLM Cultural Resource Series provides some historical context for the development of northwestern Colorado in the late 19th century, but it does not mention any specific trails. The information provided about Leadville and historical sites in the West does not seem to be directly relevant to the trail in Denver, Colorado that I am looking for. I will try to look for more specific local records or archives that might have information about trails with 19th-century structures in the Denver area.</think>
<tool_call>
{"name": "search", "arguments": {"query": ["local archives in Denver, Colorado for trails with 19th-century structure"]}}
</tool_call>
<tool_response>
[Search results summarized for readability: Returned Urban Trails, broad Western historical sites, and
an Early American Roads and Trails overview. No result matched the joint trail constraints.]
</tool_response>
<think>Despite searching local archives in Denver, Colorado, for trails with 19th-century structures, the search results have not provided a specific trail that matches the given criteria. The information provided about urban trails in Colorado Springs and the historical context of early American roads and trails does not mention a trail near Denver that fits the description. I will try to contact local historical societies or reference facilities at local libraries to get more detailed information about trails with 19th-century structures in the Denver area. These sources might have more specific records or knowledge about such trails.</think>
\end{tcblisting}

\begin{tcolorbox}[
  cafeContentBox,
  title={Stage 2b analysis: new source types, unchanged location},
  colback=CAFEAmber!3,
  colframe=CAFEAmber!65!black,
  colbacktitle=CAFEAmber!10
]
\footnotesize
The agent moves from general trail pages to historical records and local
sources, but all searches remain centered on Denver. It also fails to include
the full set of trail attributes in a single query. As a result, changing the
source type neither tests the location assumption nor yields a candidate that
satisfies the constraints.
\end{tcolorbox}

\begin{tcblisting}{
  enhanced,
  breakable,
  listing only,
  listing engine=listings,
  listing options={style=cafeCasePrompt},
  title={2c. Repeated Denver-specific searches across local sources},
  coltitle=CAFEInk,
  fonttitle=\footnotesize\bfseries,
  colback=CAFEAmber!3,
  colframe=CAFEAmber!65!black,
  colbacktitle=CAFEAmber!10,
  boxrule=0.55pt,
  arc=0.8mm,
  left=1.4mm,
  right=1.4mm,
  top=1.1mm,
  bottom=1.1mm
}
<tool_call>
{"name": "search", "arguments": {"query": ["contact local historical societies in Denver, Colorado for trails with 19th-century structure"]}}
</tool_call>
<tool_response>
[Search results summarized for readability: Returned Urban Trails, a historical-newsletter page, and BLM
history of northwestern Colorado. These offered general context only; no trail matched all constraints.]
</tool_response>
<think>After contacting local historical societies in Denver, Colorado, I have not received any specific information about a trail with a 19th-century structure, a length of about 0.50 - 1 mile, a width of 1 - 3 feet, and an elevation gain of 150 - 400 feet. The search results from the historical newsletters also did not provide any relevant information. The information provided about urban trails in Colorado Springs and the development of northwestern Colorado in the late 19th century does not mention a trail that fits the given criteria. It is possible that the trail I am looking for is not well-documented or might be a local path that is not listed in the general trail databases. I will now try to reach out to local libraries or historical societies in Denver to see if they have more detailed information or records about such a trail.</think>
<tool_call>


<tool_response>
[Search results summarized for readability: Returned generic Denver Public Library, Denver Historical
Society, and Denver Open Data Portal pages. None supplied a qualifying trail record.]
</tool_response>
<tool_response>
[Search results summarized for readability: Returned a current-events portal, the West Point Society of
Denver, and an activities calendar. None contained a trail matching the joint constraints.]
</tool_response>
<think>After reaching out to local libraries and historical societies in Denver, I still haven't found any specific information about a trail with a 19th-century structure, a length of about 0.50 - 1 mile, a width of 1 - 3 feet, and an elevation gain of 150 - 400 feet. The information provided about the Denver Public Library and Denver Historical Society does not mention any such trail. The calendar of activities for the West Point Society of Denver also does not provide relevant information. I will now try to search for more detailed historical records or local community knowledge to find the trail I am looking for.</think>
<tool_call>
{"name": "search", "arguments": {"query": ["historical records of trails in Denver, Colorado with 19th-century structure"]}}
</tool_call>
<tool_response>
[Search results summarized for readability: Returned Urban Trails, the Old Spanish Trail, and an Early
American Roads and Trails overview. None identified a Denver trail satisfying all attributes.]
</tool_response>
<think>Despite searching for historical records of trails in Denver, Colorado with 19th-century structures, the search results have not provided a specific trail that matches the given criteria. The information provided about urban trails in Colorado Springs and the historical context of early American roads and trails does not mention a trail near Denver that fits the description. The Old Spanish Trail and other historical trails are not located in Denver and do not seem to have the specific characteristics of the trail I am looking for. I will now try to search for more detailed historical records or local community knowledge to find the trail I am looking for. I will also consider checking if there are any local historical societies or libraries that have specific information about trails in the Denver area.</think>
<tool_call>
{"name": "search", "arguments": {"query": ["specific historical trails in Denver, Colorado with 19th-century structure"]}}
</tool_call>
<tool_response>
[Search results summarized for readability: Returned Urban Trails, the Old Spanish Trail, Early American
Roads and Trails, and broad Western historical-sites material. No result matched the joint constraints.]
</tool_response>
<think>After searching for specific historical trails in Denver, Colorado with 19th-century structures, the search results have not provided a specific trail that matches the given criteria. The information provided about urban trails in Colorado Springs and the historical context of early American roads and trails does not mention a trail near Denver that fits the description. The Old Spanish Trail and other historical trails are not located in Denver and do not seem to have the specific characteristics of the trail I am looking for. Given the lack of relevant information, I will now try to search for more specific local sources or historical records that might mention such a trail. I will also consider checking if there are any local historical societies or libraries that have specific information about trails in the Denver area.</think>
<tool_call>
{"name": "search", "arguments": {"query": ["local historical records of trails in Denver, Colorado with 19th-century structure"]}}
</tool_call>
<tool_response>
[Search results summarized for readability: Again returned Urban Trails, the Old Spanish Trail, and an
Early American Roads and Trails overview. No result matched the full trail description.]
</tool_response>
\end{tcblisting}

\begin{tcolorbox}[
  cafeContentBox,
  title={Stage 2c analysis: rephrasing the same search yields little new evidence},
  colback=CAFEAmber!3,
  colframe=CAFEAmber!65!black,
  colbacktitle=CAFEAmber!10
]
\footnotesize
The agent alternates among historical societies, archives, libraries, and
forums, but every query retains the Denver constraint. The results therefore
repeat the same non-matching material, while the airport distances are never
checked again. Malformed or empty tool calls add no evidence and do not alter
the search direction.
\end{tcolorbox}

\begin{tcblisting}{
  enhanced,
  breakable,
  listing only,
  listing engine=listings,
  listing options={style=cafeCasePrompt},
  title={3a. Community sources, still centered on Denver},
  coltitle=CAFEInk,
  fonttitle=\footnotesize\bfseries,
  colback=CAFERed!3,
  colframe=CAFERed!65!black,
  colbacktitle=CAFERed!9,
  boxrule=0.55pt,
  arc=0.8mm,
  left=1.4mm,
  right=1.4mm,
  top=1.1mm,
  bottom=1.1mm
}
<think>
Repeated Denver-specific searches have not produced a matching trail. Since general trail pages and historical records have failed, I will broaden the source type to local community knowledge or forums while retaining the current Denver hypothesis.
</think>
<tool_call>
{"name": "search", "arguments": {"query": ["local community knowledge or forums about trails with 19th-century structure in Denver, Colorado"]}}
</tool_call>
<tool_response>
[Search results summarized for readability: Returned Urban Trails, the SITES Community Portal, and a BIPOC
Outdoor Resource Guide. The community-oriented pages did not identify a trail matching all constraints.]
</tool_response>
<think>After searching for local community knowledge or forums about trails with 19th-century structures in Denver, Colorado, the search results have not provided a specific trail that matches the given criteria. The information provided about urban trails in Colorado Springs and the BIPOC Outdoor Resource Guide does not mention a trail near Denver that fits the description. I will continue to search for more specific local sources or historical records that might mention such a trail. I will also consider reaching out to local hiking or outdoor enthusiast groups to see if they have any information about trails with 19th-century structures in the Denver area.</think>
<tool_call>
{"name": "search", "arguments": {"query": ["local hiking or outdoor enthusiast groups in Denver, Colorado for trails with 19th-century structure"]}}
</tool_call>
<tool_response>
[Search results summarized for readability: Returned the BIPOC Outdoor Resource Guide, Urban Trails, and
BLM history of northwestern Colorado. No qualifying Denver-area trail was identified.]
</tool_response>
<think>The community and outdoor-resource results still do not identify a matching trail. I will search directly for local hiking or outdoor enthusiast groups while retaining the Denver-centered hypothesis.</think>
<tool_call>
{"name": "search", "arguments": {"query": ["contact local hiking or outdoor enthusiast groups in Denver, Colorado for trails with 19th-century structure"]}}
</tool_call>
<tool_response>
[Search results summarized for readability: Returned the BIPOC Outdoor Resource Guide, Urban Trails, and
Trail Anatomy 101. These were general resources and did not identify a trail satisfying all attributes.]
</tool_response>
<think>After contacting local hiking or outdoor enthusiast groups in Denver, Colorado, the search results have not provided a specific trail that matches the given criteria. The information provided about the BIPOC Outdoor Resource Guide and urban trails in Colorado Springs does not mention a trail near Denver that fits the description. The Trail Anatomy 101 information does not provide any relevant details about trails with 19th-century structures. Given the lack of relevant information, I will continue to search for more specific local sources or historical records that might mention such a trail. I will also consider reaching out to local historical societies or libraries in Denver to see if they have any information about trails with 19th-century structures in the area.</think>
\end{tcblisting}

\begin{tcolorbox}[
  cafeContentBox,
  title={Stage 3a analysis: broader sources do not correct the location},
  colback=CAFERed!3,
  colframe=CAFERed!65!black,
  colbacktitle=CAFERed!9
]
\footnotesize
The agent broadens the search to forums, community portals, and local outdoor
groups but continues to specify Denver in every query. These additional
searches increase the number of tool calls without recovering the Grand
Junction clue or combining all trail attributes in one query.
\end{tcolorbox}

\begin{tcblisting}{
  enhanced,
  breakable,
  listing only,
  listing engine=listings,
  listing options={style=cafeCasePrompt},
  title={3b. Further archive and community searches},
  coltitle=CAFEInk,
  fonttitle=\footnotesize\bfseries,
  colback=CAFERed!3,
  colframe=CAFERed!65!black,
  colbacktitle=CAFERed!9,
  boxrule=0.55pt,
  arc=0.8mm,
  left=1.4mm,
  right=1.4mm,
  top=1.1mm,
  bottom=1.1mm
}
<tool_call>
{"name": "search", "arguments": {"query": ["local historical societies or libraries in Denver, Colorado for trails with 19th-century structure"]}}
</tool_call>
<tool_response>
[Search results summarized for readability: Returned Public Library Book Bikes, Urban Trails, and BLM
history of northwestern Colorado. None identified a trail matching the joint constraints.]
</tool_response>
<think>After searching for local historical societies or libraries in Denver, the search results have not provided a specific trail that matches the given criteria. The information provided about public library book bikes and urban trails in Colorado Springs does not mention a trail near Denver that fits the description. The BLM Cultural Resource Series provides some historical context for the development of northwestern Colorado in the late 19th century, but it does not mention any specific trails. Given the lack of relevant information, I will continue to search for more detailed historical records or local community knowledge to find the trail I am looking for. I will also consider reaching out to local historical societies or libraries directly to inquire about trails with 19th-century structures in the Denver area.</think>
<tool_call>
{"name": "search", "arguments": {"query": ["contact local historical societies or libraries in Denver, Colorado for trails with 19th-century structure"]}}
</tool_call>
<tool_response>
[Search results summarized for readability: Returned Urban Trails, a historical-newsletter page, and BLM
history of northwestern Colorado. No qualifying Denver-area trail was identified.]
</tool_response>
<think>After contacting local historical societies or libraries in Denver, I have not received any specific information about a trail with a 19th-century structure, a length of about 0.50 - 1 mile, a width of 1 - 3 feet, and an elevation gain of 150 - 400 feet. The information provided about urban trails in Colorado Springs and the historical context of the development of northwestern Colorado in the late 19th century does not mention a trail near Denver that fits the description. The search results from the Denver Public Library and Denver Historical Society also do not provide any relevant information. Given the lack of relevant information, I will re-evaluate the plan and consider checking if there are any local community knowledge or forums that might have information about such a trail. I will also consider searching for more detailed historical records or local community knowledge to find the trail I am looking for.</think>
<tool_call>
{"name": "search", "arguments": {"query": ["local community knowledge or forums about trails with 19th-century structure in Denver, Colorado"]}}
</tool_call>
<tool_response>
[Search results summarized for readability: Returned Urban Trails, the SITES Community Portal, and a BIPOC
Outdoor Resource Guide. None supplied a trail matching all requested attributes.]
</tool_response>
<think>After searching for local community knowledge or forums about trails with 19th-century structures in Denver, Colorado, the search results have not provided a specific trail that matches the given criteria. The information provided about urban trails in Colorado Springs and the BIPOC Outdoor Resource Guide does not mention a trail near Denver that fits the description. The BLM Cultural Resource Series provides some historical context for the development of northwestern Colorado in the late 19th century, but it does not mention any specific trails. Given the lack of relevant information, I will re-evaluate the plan and consider checking if there are any local community knowledge or forums that might have information about such a trail. I will also consider searching for more detailed historical records or local community knowledge to find the trail I am looking for.</think>
<tool_call>
{"name": "search", "arguments": {"query": ["local community knowledge or forums about trails with 19th-century structure in Denver, Colorado"]}}
</tool_call>
<tool_response>
[Search results summarized for readability: Repeated Urban Trails, the SITES Community Portal, and the
BIPOC Outdoor Resource Guide. No result matched the joint trail constraints.]
</tool_response>
<think>After searching for local community knowledge or forums about trails with 19th-century structures in Denver, Colorado, the search results have still not provided a specific trail that matches the given criteria. The information provided about urban trails in Colorado Springs and the BIPOC Outdoor Resource Guide does not mention a trail near Denver that fits the description. The BLM Cultural Resource Series provides some historical context for the development of northwestern Colorado in the late 19th century, but it does not mention any specific trails. Given the lack of relevant information, I will continue to search for more detailed historical records or local community knowledge to find the trail I am looking for. I will also consider reaching out to local historical societies or libraries directly to inquire about trails with 19th-century structures in the Denver area.</think>
<plan>
1. Contact local historical societies or libraries in Denver, Colorado directly to inquire about trails with 19th-century structures.
2. Search for local community knowledge or forums that might have information about such a trail.
3. Check if there are any historical markers or plaques in the Denver area that might indicate the presence of a trail with a 19th-century structure.
</plan>

<think>
After searching for local community knowledge or forums about trails with 19th-century structures in Denver, Colorado, the search results have not provided a specific trail that matches the given criteria. I will now contact local historical societies or libraries in Denver directly to inquire about trails with 19th-century structures. Additionally, I will search for local community knowledge or forums that might have information about such a trail and check if there are any historical markers or plaques in the Denver area that might indicate the presence of a trail with a 19th-century structure.
</think>
<plan>
1. Contact local historical societies or libraries in Denver, Colorado directly to inquire about trails with 19th-century structures.
2. Search for local community knowledge or forums that might have information about such a trail.
3. Check if there are any historical markers or plaques in the Denver area that might indicate the presence of a trail with a 19th-century structure.
</plan>

<think>
I have contacted local historical societies or libraries in Denver, Colorado, and searched for local community knowledge or forums, but the information provided does not mention a specific trail with a 19th-century structure, a length of about 0.50 - 1 mile, a width of 1 - 3 feet, and an elevation gain of 150 - 400 feet. The search results from the urban trails in Colorado Springs and the BIPOC Outdoor Resource Guide do not mention a trail near Denver that fits the description. The BLM Cultural Resource Series provides some historical context for the development of northwestern Colorado in the late 19th century, but it does not mention any specific trails. I will now check if there are any historical markers or plaques in the Denver area that might indicate the presence of a trail with a 19th-century structure.
</think>
\end{tcblisting}

\begin{tcolorbox}[
  cafeContentBox,
  title={Stage 3b analysis: procedural changes do not update the search},
  colback=CAFERed!3,
  colframe=CAFERed!65!black,
  colbacktitle=CAFERed!9
]
\footnotesize
The agent repeats searches across archives, libraries, and community sources
and again receives no matching result. Yet it does not treat these failures as
evidence against Denver or return to the airport distances. The wording and
source type change, but the underlying location assumption remains fixed.
\end{tcolorbox}

\begin{tcblisting}{
  enhanced,
  breakable,
  listing only,
  listing engine=listings,
  listing options={style=cafeCasePrompt},
  title={3c. Hiking-group and historical-marker searches},
  coltitle=CAFEInk,
  fonttitle=\footnotesize\bfseries,
  colback=CAFERed!3,
  colframe=CAFERed!65!black,
  colbacktitle=CAFERed!9,
  boxrule=0.55pt,
  arc=0.8mm,
  left=1.4mm,
  right=1.4mm,
  top=1.1mm,
  bottom=1.1mm
}
<tool_call>
{"name": "search", "arguments": {"query": ["historical markers or plaques in Denver, Colorado for trails with 19th-century structure"]}}
</tool_call>
<tool_response>
[Search results summarized for readability: Returned Urban Trails, broad Western historical sites, and a
History of Pueblo page. None identified a Denver trail matching the full description.]
</tool_response>
<think>After checking for historical markers or plaques in Denver, Colorado, the search results have not provided a specific trail that matches the given criteria. The information provided about urban trails in Colorado Springs and the historical context of the development of northwestern Colorado in the late 19th century does not mention a trail near Denver that fits the description. The Historical Sites That Tell The Stories Of The West and the History of Pueblo do not provide any relevant information about a trail with a 19th-century structure. Given the lack of relevant information, I will continue to search for more detailed historical records or local community knowledge to find the trail I am looking for. I will also consider reaching out to local historical societies or libraries directly to inquire about trails with 19th-century structures in the Denver area.</think>
<tool_call>
{"name": "search", "arguments": {"query": ["local historical records or community knowledge about trails with 19th-century structure in Denver, Colorado"]}}
</tool_call>
<tool_response>
[Search results summarized for readability: Returned Urban Trails, a Columbus/La Raza Park history page,
and BLM history of northwestern Colorado. No result matched the joint trail constraints.]
</tool_response>
<think>After searching for local historical records or community knowledge about trails with 19th-century structures in Denver, Colorado, the search results have not provided a specific trail that matches the given criteria. The information provided about urban trails in Colorado Springs and the historical context of the development of northwestern Colorado in the late 19th century does not mention a trail near Denver that fits the description. The Historical Sites That Tell The Stories Of The West and the History of Pueblo do not provide any relevant information about a trail with a 19th-century structure. The information about Columbus Park (La Raza Park) does not mention any trails with the required characteristics. Given the lack of relevant information, I will continue to search for more detailed historical records or local community knowledge to find the trail I am looking for. I will also consider reaching out to local historical societies or libraries directly to inquire about trails with 19th-century structures in the Denver area.</think>
<tool_call>
{"name": "search", "arguments": {"query": ["local historical societies or libraries in Denver, Colorado for trails with 19th-century structure"]}}
</tool_call>
<tool_response>
[Search results summarized for readability: Returned Public Library Book Bikes, Urban Trails, and BLM
history of northwestern Colorado. None contained a qualifying trail record.]
</tool_response>
<think>After searching for local historical societies or libraries in Denver, the search results have not provided a specific trail that matches the given criteria. The information provided about public library book bikes and urban trails in Colorado Springs does not mention a trail near Denver that fits the description. The BLM Cultural Resource Series provides some historical context for the development of northwestern Colorado in the late 19th century, but it does not mention any specific trails. Given the lack of relevant information, I will continue to search for more detailed historical records or local community knowledge to find the trail I am looking for. I will also consider reaching out to local historical societies or libraries directly to inquire about trails with 19th-century structures in the Denver area.</think>
\end{tcblisting}

\begin{tcolorbox}[
  cafeContentBox,
  title={Stage 3c analysis: the search widens around the same location},
  colback=CAFERed!3,
  colframe=CAFERed!65!black,
  colbacktitle=CAFERed!9
]
\footnotesize
The agent next consults hiking groups, community sources, and historical-marker
pages. These searches remain Denver-specific and produce no evidence for a
trail that meets the stated constraints. Changing source types therefore
introduces no alternative location.
\end{tcolorbox}

\begin{tcblisting}{
  enhanced,
  breakable,
  listing only,
  listing engine=listings,
  listing options={style=cafeCasePrompt},
  title={3d. Final Denver-specific searches},
  coltitle=CAFEInk,
  fonttitle=\footnotesize\bfseries,
  colback=CAFERed!3,
  colframe=CAFERed!65!black,
  colbacktitle=CAFERed!9,
  boxrule=0.55pt,
  arc=0.8mm,
  left=1.4mm,
  right=1.4mm,
  top=1.1mm,
  bottom=1.1mm
}
<tool_call>
{"name": "search", "arguments": {"query": ["contact local historical societies or libraries in Denver, Colorado for trails with 19th-century structure"]}}
</tool_call>
<tool_response>
[Search results summarized for readability: Returned Urban Trails, a historical-newsletter page, and BLM
history of northwestern Colorado. None identified a trail matching all requested attributes.]
</tool_response>
<think>After contacting local historical societies or libraries in Denver, I have not received any specific information about a trail with a 19th-century structure, a length of about 0.50 - 1 mile, a width of 1 - 3 feet, and an elevation gain of 150 - 400 feet. The information provided about public library book bikes and urban trails in Colorado Springs does not mention a trail near Denver that fits the description. The BLM Cultural Resource Series provides some historical context for the development of northwestern Colorado in the late 19th century, but it does not mention any specific trails. Given the lack of relevant information, I will continue to search for more detailed historical records or local community knowledge to find the trail I am looking for. I will also consider reaching out to local hiking or outdoor enthusiast groups to see if they have any information about trails with 19th-century structures in the Denver area.</think>
<tool_call>
{"name": "search", "arguments": {"query": ["local hiking or outdoor enthusiast groups in Denver, Colorado for trails with 19th-century structure"]}}
</tool_call>
<tool_response>
[Search results summarized for readability: Returned the BIPOC Outdoor Resource Guide, Urban Trails, and
BLM history of northwestern Colorado. No qualifying Denver-area trail was identified.]
</tool_response>
<think>After contacting local hiking or outdoor enthusiast groups in Denver, the search results have not provided a specific trail that matches the given criteria. The information provided about BIPOC outdoor resources and urban trails in Colorado Springs does not mention a trail near Denver that fits the description. The BLM Cultural Resource Series provides some historical context for the development of northwestern Colorado in the late 19th century, but it does not mention any specific trails. Given the lack of relevant information, I will continue to search for more detailed historical records or local community knowledge to find the trail I am looking for. I will also consider reaching out to local historical societies or libraries directly to inquire about trails with 19th-century structures in the Denver area.</think>
<tool_call>
{"name": "search", "arguments": {"query": ["contact local historical societies or libraries in Denver, Colorado for trails with 19th-century structure"]}}
</tool_call>
<tool_response>
[Search results summarized for readability: Again returned Urban Trails, a historical-newsletter page, and
BLM history of northwestern Colorado. No result matched the joint trail constraints.]
</tool_response>
\end{tcblisting}

\begin{tcolorbox}[
  cafeContentBox,
  title={Stage 3d analysis: the search ends without revisiting its initial assumption},
  colback=CAFERed!3,
  colframe=CAFERed!65!black,
  colbacktitle=CAFERed!9
]
\footnotesize
Later searches return the same categories of irrelevant pages. The agent
repeatedly notes that no candidate satisfies the constraints but never revisits
Grand Junction or re-evaluates its Denver assumption. It thus reaches the end
of the search with the original location error intact.
\end{tcolorbox}

\begin{tcblisting}{
  enhanced,
  breakable,
  listing only,
  listing engine=listings,
  listing options={style=cafeCasePrompt},
  title={4. Final answer without supporting evidence},
  coltitle=CAFEInk,
  fonttitle=\footnotesize\bfseries,
  colback=CAFERed!3,
  colframe=CAFERed!65!black,
  colbacktitle=CAFERed!9,
  boxrule=0.55pt,
  arc=0.8mm,
  left=1.4mm,
  right=1.4mm,
  top=1.1mm,
  bottom=1.1mm
}
<think>
After exhausting these Denver-centered sources without finding a match, I will provide the most plausible candidate under my current location assumption.
</think>
<answer>
Old Elitch's Park Trail
</answer>
\end{tcblisting}

\begin{tcolorbox}[
  cafeContentBox,
  title={Stage 4 analysis: the final answer lacks retrieved support},
  colback=CAFERed!3,
  colframe=CAFERed!65!black,
  colbacktitle=CAFERed!9
]
\footnotesize
After the Denver-centered searches fail, the agent outputs
\emph{Old Elitch's Park Trail}. No preceding result mentions this trail or
provides evidence that it matches the requested attributes. The final answer
is therefore unsupported and incorrect (EM \(=0\)).
\end{tcolorbox}

\begin{tcolorbox}[
  enhanced,
  colback=CAFEBlue!4,
  colframe=CAFEBlue!65!black,
  boxrule=0.65pt,
  arc=0.8mm,
  left=1.4mm,
  right=1.4mm,
  top=1.1mm,
  bottom=1.1mm,
  title={Direct comparison: search cost and recovery},
  colbacktitle=CAFEBlue!10,
  coltitle=CAFEInk,
  fonttitle=\footnotesize\bfseries
]
\footnotesize
\textbf{BASE.} Counting the explicit JSON \texttt{search} invocations shown
here, BASE makes \(25\) calls and issues \(26\) query strings because its first
call batches Denver and Chicago. Eight strings are verbatim repeats, and most
others only change the source type. All \(25\) calls retain Denver; none
rechecks the airport distances or combines all four trail attributes. Four
malformed or empty tool-call turns add no evidence. BASE ultimately returns
\emph{Old Elitch's Park Trail} without support (EM \(=0\)).

\textbf{CAFE.} CAFE makes two search calls: an initial airport-distance search
and one post-feedback query combining Colorado geography with all four trail
attributes. The latter retrieves \emph{Old Kiln Trail} near Grand Junction
(EM \(=1\)). BASE therefore uses \(12.5\times\) as many search calls
(\(25\) vs.\ \(2\)) yet fails. In this case, recovery follows from
reformulating the query after feedback rather than extending the same loop.
\end{tcolorbox}

\end{document}